%% file: main.tex
\documentclass[10pt]{article}

\usepackage[
    letterpaper,
    top=0.60in,
    bottom=0.75in,
    left=0.85in,
    right=0.85in
]{geometry}

\usepackage{times}
\usepackage{microtype}

\usepackage{amsmath}
\usepackage{amssymb}
\usepackage{amsthm}
\usepackage{mathtools}
\usepackage{bm}

\usepackage{algorithm}
\usepackage{algpseudocode}

\usepackage{graphicx}
\usepackage{subcaption}
\usepackage{booktabs}
\usepackage{tabularx}
\usepackage{makecell}
\usepackage{adjustbox}
\usepackage[most]{tcolorbox}
\usepackage{float}
\usepackage{caption}

\usepackage{xcolor}
\usepackage[numbers,sort&compress]{natbib}

\usepackage[
    colorlinks=true,
    citecolor=blue,
    linkcolor=blue,
    urlcolor=blue
]{hyperref}

\usepackage[nameinlink,capitalize]{cleveref}
\usepackage{url}
\usepackage{titlesec}
\usepackage{titling}

\pretitle{\centering\normalfont\LARGE}
\posttitle{\par\vspace{0.5em}}

\preauthor{\centering\normalfont\large}
\postauthor{\par\vspace{-0.7em}}

\titleformat{\section}
  {\large\bfseries}
  {\thesection}
  {0.75em}
  {}

\titleformat{\subsection}
  {\normalsize\bfseries}
  {\thesubsection}
  {0.75em}
  {}

\titleformat{\subsubsection}
  {\normalsize\bfseries}
  {\thesubsubsection}
  {0.75em}
  {}

\titlespacing*{\section}
  {0pt}
  {2.5ex plus 0.5ex minus 0.2ex}
  {1.0ex}

\titlespacing*{\subsection}
  {0pt}
  {2.0ex plus 0.4ex minus 0.2ex}
  {0.8ex}

\titlespacing*{\subsubsection}
  {0pt}
  {1.5ex plus 0.3ex minus 0.2ex}
  {0.6ex}

\renewenvironment{abstract}
{
    \begin{center}
    \begin{minipage}{0.86\textwidth}
    \small
    \begin{center}
    \bfseries Abstract
    \end{center}
    \vspace{-0.4em}
}
{
    \end{minipage}
    \end{center}
}

\newtheorem{theorem}{Theorem}
\newtheorem{proposition}[theorem]{Proposition}

\newtheorem{lemma}[theorem]{Lemma}

\theoremstyle{definition}
\newtheorem{definition}[theorem]{Definition}
\newtheorem{assumption}[theorem]{Assumption}

\theoremstyle{remark}
\newtheorem{remark}[theorem]{Remark}

\newcommand{\R}{\mathbb{R}}

\newcommand{\dd}{\mathrm{d}}

\newcommand{\inv}{^{-1}}

\newcommand{\norm}[1]{\left\lVert #1 \right\rVert}

\DeclareMathOperator*{\argmin}{arg\,min}
\DeclareMathOperator{\SPD}{SPD}
\DeclareMathOperator{\Id}{Id}

\DeclareMathOperator{\di}{div}

\crefname{equation}{Eq.}{Eqs.}
\Crefname{equation}{Equation}{Equations}

\crefname{figure}{Fig.}{Figs.}
\Crefname{figure}{Figure}{Figures}

\crefname{table}{Table}{Tables}
\Crefname{table}{Table}{Tables}

\crefname{section}{Section}{Sections}
\Crefname{section}{Section}{Sections}

\crefname{theorem}{Theorem}{Theorems}
\Crefname{theorem}{Theorem}{Theorems}

\crefname{proposition}{Proposition}{Propositions}
\Crefname{proposition}{Proposition}{Propositions}

\crefname{lemma}{Lemma}{Lemmas}
\Crefname{lemma}{Lemma}{Lemmas}

\title{
Finsler Flow Matching:\\
Dynamics-Aware Geodesic Interpolation \\ for Single-Snapshot Trajectory Inference
}

\author{
Niklas Canova
\qquad
Jonas Simon Fleck
\\[0.25em]
\small Institute of Human Biology (IHB), Roche, Basel, Switzerland
\\[0.1em]
\small \texttt{niklas.canova@roche.com}
}

\date{}

\begin{document}

\maketitle
\begingroup
\renewcommand{\thefootnote}{}
\footnotetext{Code: \url{https://github.com/mlbelab/FinslerFlowMatching}}
\endgroup
\input{sections/abstract}

\begin{figure}[H]
    \centering
    \makebox[\textwidth][c]{%
        \includegraphics[
            width=\textwidth,
            keepaspectratio
        ]{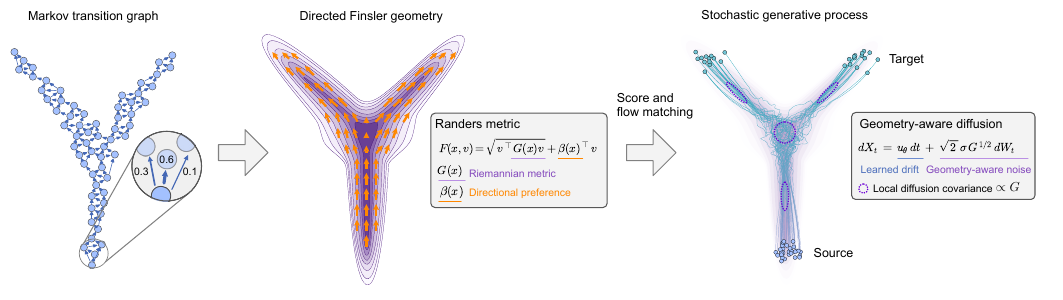}
    }
    \caption{
        \textbf{Overview of Finsler Flow Matching (FFM).}
        A discrete Markov transition graph is converted into a continuous
        directed Finsler geometry, which guides score and flow matching
        to define a stochastic generative model with geometry-aware diffusion.
    }
    \label{fig:ffm_schematic}
\end{figure}

\input{sections/introduction}
\input{sections/preliminaries}

\input{sections/related_work}

\input{sections/theoretical_contributions}
\input{sections/experiments}

\input{sections/discussion}

\subsection*{AI use statement}

We used generative AI tools to assist with code generation, in particular
code for generating synthetic datasets, as well as sentence reformulation
and language polishing throughout the manuscript. Furthermore, we used
deep research tools to discover papers. In particular, these tools identified
connections to FW theory, which we subsequently investigated, verified, and
developed in the context of our method. AI tools were further used to assist
with completing and checking mathematical proofs. The theorem statements
and underlying proof strategies were developed by the authors, and all
AI-assisted derivations were independently checked for correctness. We
reviewed all AI-generated output and take full responsibility for all content
in this paper.

\subsection*{Reproducibility Statement}

We provide a code artifact containing the implementation of Finsler Flow
Matching, preprocessing and transition-kernel construction, synthetic-data
generation, training and evaluation scripts, benchmark-specific
configurations, hyperparameter settings, and the records used to generate
the reported tables and figures.

Software dependencies are pinned in \texttt{requirements-lock.txt} to
facilitate reconstruction of the experimental environment. Experimental
architecture, optimization settings, training budgets, validation procedures,
and numbers of random seeds are reported in the appendix.

Dataset sources and preprocessing procedures are documented for each
benchmark, and validation and test populations are kept disjoint during model
and hyperparameter selection. We additionally provide the configurations used
for all reported methods and hyperparameter searches, including the
data-dependent scaling of the Finsler regularization parameters.

Evaluation settings, including any subsampling used for optimal-transport
calculations, are recorded alongside the experiments. Because some GPU
operations are nondeterministic, repeated training runs are not guaranteed
to be bitwise identical; we therefore report results across multiple random
seeds. The complete repository will be made publicly available upon
publication.

\clearpage

\bibliographystyle{plainnat}
\bibliography{sections/references}

\clearpage
\appendix

\input{sections/appendix}

\end{document}

%% file: sections/abstract.tex
\begin{abstract}
Single-cell snapshot data can resolve a continuum of cellular states but do not uniquely determine the dynamics governing transitions between them. However, additional dynamical information can often be encoded in a cell-cell Markov transition kernel. Existing generative approaches for single cell trajectory inference either infer transport only from population marginals, impose a symmetric geometry on the state space, or incorporate directionality through a single velocity vector at each observed state. We introduce Finsler Flow Matching (FFM), a framework for learning continuous stochastic dynamics from discrete Markov transition graphs. We use the first and second local moments to construct a Finsler structure motivated by the Freidlin--Wentzell action, where the second moment determines anisotropic accessibility and the first moment introduces a preferred direction of motion. We learn neural approximations of the resulting directed geodesics, use their Finsler cost to construct source-target couplings, and define geometry-aware stochastic conditional paths that can be distilled into a continuous generative process through simulation-free score and flow matching. Across synthetic and single-cell trajectory inference benchmarks, FFM improves recovery of withheld intermediate populations, particularly when the transition dynamics are strongly directional or anisotropic. Our results provide a principled route from discrete transition probabilities to continuous generative dynamics while retaining both directional and diffusive structure.
\end{abstract}

%% file: sections/introduction.tex
\section{Introduction}
Understanding how cells change their state in response to intrinsic programs and extrinsic perturbations is a central challenge in biology. Such transitions underlie development and differentiation, but also cellular reprogramming, disease progression, and responses to therapeutic intervention. Single-cell genomics methods have made it possible to probe these processes at high resolution by resolving a continuum of cellular states in a single cross-sectional experiment \citep{trapnellDynamicsRegulatorsCell2014}. However, because single-cell sequencing methods are destructive, individual cells cannot be followed through time and transitions between observed states are not directly measured. As a result, the continuous dynamics underlying cell state changes can generally not be determined from these measurements alone \citep{weinrebFundamentalLimitsDynamic2018}.

Several experimental and computational strategies can provide additional information about developmental directionality and dynamics. RNA velocity infers the direction of local transcriptional change {\citep{lamannoRNAVelocitySingle2018,bergenGeneralizingRNAVelocity2020a}, lineage tracing tracks ancestry \citep{weinrebLineageTracingTranscriptional2020,wagnerLineageTracingMeets2020,mckennaWholeorganismLineageTracing2016} and metabolic or fluorescent labelling can introduce an absolute timescale \citep{battichSequencingMetabolicallyLabeled2020,caoScifateCharacterizesDynamics2020,gehartIdentificationEnteroendocrineRegulators2019}. Measurements collected across multiple experimental time points also carry temporal information, although the correspondence between cells at successive time points is lost upon sampling and must itself be inferred \citep{schiebingerOptimalTransportAnalysisSingleCell2019}. Despite their different origins, these sources of dynamical information can often be translated into a common representation, a cell-cell Markov transition kernel, whose entries quantify the likelihood of transitions between observed cellular states \citep{langeCellRankDirectedSinglecell2022a, weilerCellRank2Unified2024}. Such kernels have proven highly effective for inferring terminal states, lineage relationships, and fate probabilities \citep{fleckInferringPerturbingCell2023, langeMappingLineagetracedCells2024b}. However, the resulting dynamics remain defined on the finite graph of observed cells. A Markov kernel encodes transition preferences between observed states, but this does not specify how the system evolves continuously through state space, particularly between or beyond the observed cells. 

Generative modeling of population dynamics provides a complementary approach \citep{neklyudovActionMatchingLearning2023}. Given samples at multiple time points, optimal transport formulations recover a coupling between the observed marginals \citep{schiebingerOptimalTransportAnalysisSingleCell2019,cuturiSinkhornDistancesLightspeed2013a}, while flow and Schrödinger Bridge matching provide simulation-free frameworks for learning continuous dynamics \citep{tongSimulationFreeSchrodingerBridges2024a, lipmanFlowMatchingGenerative2022a,rohbeckModelingComplexSystema, tongImprovingGeneralizingFlowbased2024a}. Some variants additionally incorporate the geometry of the state or distribution space \citep{chenFlowMatchingGeneral2023, havivWassersteinFlowMatching2025a}. Metric Flow Matching replaces Euclidean interpolation with geodesics under a learned Riemannian metric \citep{kapusniakMetricFlowMatchinga}, capturing the anisotropic geometry of the data manifold but not a preferred direction of motion. Curly Flow Matching incorporates RNA velocity through a nonzero-drift reference process and can therefore capture directed and non-gradient dynamics \citep{petrovicCurlyFlowMatching2025a}, but takes as input a single drift vector at each state, and therefore cannot exploit the second-order structure of a local transition distribution.

In this paper, we use a discrete Markov transition graph to define a continuous dynamical prior for endpoint-constrained generative transport. We introduce Finsler Flow Matching (FFM), which converts the directional and diffusive structure of a Markov transition graph into a smooth Finsler geometry and uses this geometry to learn continuous stochastic transport through score and flow matching. Under the additional assumption that the Markov chain discretizes an Itô diffusion, the graph-derived local moments recover its drift and diffusion. Freidlin–Wentzell \citep{freidlinRandomPerturbationsDynamical2012a, heymannGeometricMinimumAction2008a} theory then motivates a Randers metric \citep{baoIntroductionRiemannFinslerGeometry2000a}, in which local diffusion determines anisotropic accessibility and drift introduces directional preference. We then learn a neural path interpolant to approximate directed Randers geodesics and use the resulting path cost to construct source--target
couplings consistent with the transition dynamics. Finally, we use these couplings and geodesic interpolants to construct tractable conditional
probability paths and learn an endpoint-constrained stochastic generative process through simulation-free score and flow matching \citep{tongSimulationFreeSchrodingerBridges2024a}, whose local noise geometry is shaped by the graph-derived second moment.

To summarize, our main contributions are:
\begin{itemize}
    \item \textbf{A principled construction of Finsler geometry from Markov graphs.} Under a short-time Itô diffusion model, we show that the first and second local moments of a cell-cell Markov kernel recover the drift and diffusion of an underlying stochastic process. We use these estimates to derive a Randers geometry from the Freidlin-Wentzell minimum-action principle \citep{freidlinRandomPerturbationsDynamical2012a}, jointly encoding anisotropic accessibility and directional dynamics.
    \item \textbf{A framework for learning continuous stochastic dynamics from discrete Markov graphs.} We introduce Finsler Flow Matching, which learns a neural path interpolant under the induced Finsler geometry, uses directed geodesic costs to construct source-target couplings, and defines stochastic conditional paths whose local noise structure is shaped by the graph-derived diffusion. These paths are then distilled into a continuous generative process through simulation-free score and flow matching.
    \item \textbf{Flexible inference of single-cell trajectory dynamics.} By using cell-cell Markov graphs as a common representation of transition dynamics, FFM can incorporate information derived from various experimental and computational sources, including RNA velocity, lineage tracing, labeling experiments, or temporally sampled populations within a unified generative framework. We show that our model outperforms previous methods in intermediate state prediction on various benchmarks.
\end{itemize}

%% file: sections/preliminaries.tex
\section{Preliminaries}
In this section, we establish the background on Finsler geometry, optimal transport, and continuous generative modeling via flow and score matching. More information can be found in \cref{appendix:Finsler}
\paragraph{Finsler Geometry:} 
Riemannian geometry measures local distance via a symmetric metric tensor $G(x)$. Finsler geometry generalizes this to a direction-dependent, asymmetric cost. We introduce a special class of Finsler structures called a \emph{Randers metric}:
\begin{equation*}
F(x, v) = \sqrt{v^\top G(x) v} + \beta(x)^\top v,
\end{equation*}
where $G(x) \in \SPD(d)$ is a Riemannian metric tensor and $\beta(x)$ is a $1$-form giving  directionality. For an absolutely continuous path $\gamma: [0,1] \to \mathbb{R}^d$, the Finsler length $L_F(\gamma)$ and energy $E_F(\gamma)$ functionals are defined as:
\begin{equation*}
L_F(\gamma) \coloneqq \int_0^1 F(\gamma(t), \dot{\gamma}(t))  \dd t, \qquad E_F(\gamma) \coloneqq \int_0^1 F(\gamma(t), \dot{\gamma}(t))^2  \dd t.
\end{equation*}
Curves minimizing $L_F$ are called \emph{(Finsler) geodesics} and are locally shortest paths connecting two points. The \emph{Finsler geodesic distance} $c_F(x_0, x_1)$ is defined as:
\begin{equation}\label{eq:finsler-geodesic-cost}
c_F(x_0, x_1) \coloneqq \inf_{\substack{\gamma(0)=x_0 \ \gamma(1)=x_1}} \int_0^1 F(\gamma(t), \dot{\gamma}(t)) \mathrm{d}t.
\end{equation}
\paragraph{Optimal Transport and Schrödinger Bridge Problem:} Optimal Transport (OT) aims to transport a probability distribution into another one in the most cost efficient way \citep{villaniOptimalTransport2009a}. Given source and target densities $p_0, p_1 \in \mathcal{P}(\mathbb{R}^d)$ and a cost function $c: \mathbb{R}^d \times \mathbb{R}^d \to \mathbb{R}_{\geq 0}$, the \emph{Kantorovich problem} seeks an optimal transport plan $\pi \in \Pi(p_0, p_1)$ minimizing total cost:
\begin{equation*}
\pi^\star = \arg\inf_{\pi \in \Pi(p_0, p_1)} \int_{\mathbb{R}^d \times \mathbb{R}^d} c(x_0, x_1) \mathrm{d}\pi(x_0, x_1),
\end{equation*}
where $\Pi(p_0, p_1)$ is the set of joint distributions with marginals $p_0$ and $p_1$. Using the squared Finsler geodesic distance as cost function and the elementary fact
$c_F(x_0,x_1)^2 = \inf_{\gamma: \gamma(0)=x_0, \gamma(1)=x1} E_F(\gamma)$ \footnote{It follows from the Cauchy-Schwarz inequality and the fact that minimizers of the Energy are constant speed geodesics.} implies that
\begin{align} \label{eq:OTFinsler}
W_{2,F}^2(p_0,p_1) = \inf_{\pi \in \Pi(p_0, p_1)} \inf_{\gamma: \gamma(0)=x_0, \gamma(1)=x_1}\int_{\mathbb{R}^d \times \mathbb{R}^d}  \int_0^1 F^2(\gamma(t), \dot \gamma(t))\text{ } \dd t \text{ }\mathrm{d}\pi(x_0, x_1).
\end{align}
A stochastic counterpart to the OT problem is given by the Schrödinger Bridge Problem (SBP) \citep{leonardSurveySchrodingerProblem2013a}. Given a reference measure on  $\Omega=C([0,1],\mathbb{R}^d)$ of a stochastic process $R$, the SBP searches for the closest path measure to $R$ that has endpoint marginals given by $p_0$ and $p_1$. More precisely
\begin{align}
     Q^* = \argmin_{Q:Q_0=p_0,;Q_1=p_1}
\mathrm{KL}(Q\Vert R).
\end{align}
A classical result due to \citet{follmerRandomFieldsDiffusion1988a} tells us that the solution can be decomposed as
\begin{align} \label{eq: Föllmer}
Q^*= \int R^{x_0,x_1}\dd \pi^*(x_0,x_1)\text{, where } \pi^*=\argmin_{\pi\in\Pi(p_0,p_1)}\mathrm{KL}\left(\pi\Vert R_{01}\right) ,
\end{align}
where $R^{x_0,x_1}=R(\cdot\mid Y_0=x_0,Y_1=x_1)$ denotes the reference process conditioned on its endpoints and $R_{01}$ its endpoint law. For a Brownian reference process $dY_t=\sigma_B dW_t$, one can see that
$R_t^{x_0,x_1}=\mathcal{N}\left((1-t)x_0+t x_1,\sigma^2t(1-t)I\right).$
For non-Brownian reference processes, such as SDEs with anisotropic diffusion and non-trivial drift
\begin{align*}
dY_t=b_0(Y_t),dt+\Sigma(Y_t)dW_t,
\end{align*}
the conditional bridges are in general neither Gaussian nor available in closed form. In the small noise regime, their most-probable transition paths are instead characterized by minimizers of the Freidlin-Wentzell action, a connection that we use in \cref{sec: appendic Theoretical}.
\paragraph{Score and Flow Matching:} \label{para:Score and Flow}
Flow Matching (FM) is a simulation-free method for training continuous-time generative models. Given a target probability path $p_t(x)$ pushed forward by an unknown vector field $u_t(x)$, FM fits a neural vector field $v_\theta(t, x)$ via a regression objective:
\begin{equation}
\mathcal{L}_{\mathrm{FM}}(\theta) = \int_0^1 \mathbb{E}_{x \sim p_t(x)} \left[ |v_\theta(t, x) - u_t(x)|^2 \right] \mathrm{d}t.
\end{equation}
Conditional Flow Matching (CFM) circumvents the intractability of the marginal vector field $u_t(x)$ by conditioning on endpoint pairs $(x_0, x_1) \sim q(x_0, x_1)$ \citep{lipmanFlowMatchingGenerative2022a, albergoBuildingNormalizingFlows2023a}. The network is trained on tractable conditional vector fields $v_t(x \mid x_0, x_1)$ pushing probability paths along conditional trajectories $\mu_t(x_0, x_1)$. FM is often combined with minibatch optimal transport to pair source and target samples, enabling it to approximate dynamic optimal transport \citep{tongImprovingGeneralizingFlowbased2024a}.
Following the approach of \citet{tongSimulationFreeSchrodingerBridges2024a} one may jointly training the flow and score through the combined objective
\begin{align*}
\mathcal{L}_{\mathrm{[SF]}^2\mathrm{M}}(\theta)
=\mathcal{L}_{\mathrm{CFM}}(\theta)
+ \lambda \mathcal{L}_{\mathrm{CSM}}(\theta).
\end{align*}
The learned flow $v_\theta$ and score $s_\psi$ can be used with an arbitrary matrix-valued diffusion schedule $\mathbf g:[0,1]\to\mathbb R^{d\times d}$ to define the stochastic generative process
\begin{align}\label{eq: SDE prelim}
    dX_t
    =
    \left[
        v_\theta(t,X_t)
        +
        \frac{1}{2}
        \mathbf g(t)\mathbf g(t)^\top
        s_\psi(t,X_t)
    \right]dt
    +
    \mathbf g(t)\,dW_t,
\end{align}
which satisfies $X_t \sim p_t $ for all $t \in [0,1].$

%% file: sections/related_work.tex
\section{Related Work}
\paragraph{Single-cell trajectory inference.} OT methods infer developmental trajectories from populations observed at multiple time points by coupling successive snapshot distributions \citep{schiebingerOptimalTransportAnalysisSingleCell2019, tongTrajectoryNetDynamicOptimal2020, bunneProximalOptimalTransport2022, zhangLearningStochasticDynamics2025a}. Relatedly, \citet{huguetManifoldInterpolatingOptimalTransporta} combine manifold-aware OT with neural differential equations to learn continuous cellular dynamics from snapshot data. Complementary approaches such as CellRank construct cell-cell Markov transition kernels from RNA velocity, lineage information, and other dynamical measurements to infer fate probabilities and developmental relationships \citep{langeCellRankDirectedSinglecell2022a,weilerCellRank2Unified2024}. FFM builds on these representations by using local transition statistics to construct a smooth directed geometry and continuous generative dynamics beyond the observed cells.

\paragraph{Stochastic generative transport.} Simulation-free Schrödinger bridge methods combine score and flow matching to learn stochastic dynamics between endpoint distributions \citep{tongSimulationFreeSchrodingerBridges2024a}. Related work studies closed-form Gaussian Schrödinger bridges and smooth Schrödinger bridges for trajectory inference \citep{bunneSchrodingerBridgeGaussian2023, hongTrajectoryInferenceSmooth2025}. FFM builds on this framework by using the transition-derived geometry to determine both the interpolating paths and the local anisotropic noise structure.

\paragraph{Geometric and directional flow matching.} \citet{kapusniakMetricFlowMatchinga} replace Euclidean conditional paths with neural approximations of geodesics under a data-dependent Riemannian metric. FFM extends this idea to a directed Randers geometry derived from local transition statistics.  Relatedly, \citet{bambergerCarreChampFlow2026} uses local Markov kernels to construct anisotropic covariance for geometry-aware flow matching, whereas FFM additionally incorporates first-moment directionality through a Randers geometry and geodesic coupling.\citet{petrovicCurlyFlowMatching2025a} uses a nonzero-drift reference process to model directed and non-gradient dynamics from local velocity fields. FFM instead derives directionality from the first moment of a transition kernel and uses its second moment to capture anisotropic transition structure.

\paragraph{Lineage-informed Finsler geometry.} \citet{zweigLearningLineageguidedGeodesics2026a} recently use Finsler geometry to incorporate directed lineage-tree priors into trajectory inference, penalizing motion inconsistent with a prescribed adjacency between cell types. FFM instead derives its Finsler geometry directly from cell-cell transition probabilities, using the probabilistic transition graph to encode both directional preference and local transition accessibility rather than relying on a categorical lineage prior.

%% file: sections/theoretical_contributions.tex
\section{Finsler Flow Matching}
\paragraph{Problem formulation.}
Let $D=\{x_i\}_{i=1}^N\subseteq \mathbb{R}^d$ be a collection of observations sampled from an unknown distribution $p$. Let
$D_0=\{x_i^0\}_{i=1}^{N_0}$ and $D_1=\{x_i^1\}_{i=1}^{N_1}$
denote two subpopulations from source and target distributions $p_0$ and $p_1$. In addition, we are given a row-stochastic transition
matrix $P\in\mathbb{R}^{N\times N}$, whose entries represent transition
likelihoods between observed states.
\\
\\
Our goal is to infer a transport mechanism from $p_0$ to $p_1$ and a family of intermediate distributions $(p_t)_{t\in[0,1]}$ that are consistent with the dynamics encoded by $P$. More formally we can define this as an OT problem of the form \ref{eq:OTFinsler}
\begin{align} \label{eq:OT Finsler Problem}
\begin{aligned}
&(\pi^*, \gamma^*)= \argmin_{(\pi,\gamma)} \int_{\mathbb{R}^d \times \mathbb{R}^d}  \int_0^1 F_P^2(\gamma(t), \dot \gamma(t))\text{ } \dd t\text{ } \mathrm{d}\pi(x_0, x_1) \\&\text{ where }\pi \in \Pi(p_0, p_1)\text{ and }\gamma(0)=x_0, \gamma(1)=x_1 
\end{aligned}
\end{align}
and $F_P(x, v) = \sqrt{v^\top G_P(x) v} + \beta_P(x)^\top v$ is a $P$-derived Randers metric (Definition~\ref{def:Pderived}). The corresponding intermediate distributions are induced by transporting
the optimal coupling along the optimal paths.
\\
\\
From now on we consider the special case where the $P$-derived Randers metric is constructed from spatial moments of $P$. More precisely we define at each observed state $x_i$ moments as:
\begin{align}
m(x_i)
\coloneqq
\sum_{j=1}^N P_{ij}(x_j-x_i), \text{ and }
C(x_i)
\coloneqq
\sum_{j=1}^N
P_{ij}(x_j-x_i)(x_j-x_i)^\top .
\end{align}
\footnote{At finite graph scale \(C=\operatorname{Cov}_{P}(\Delta x)+mm^\top\), we interpret \(C\) as an anisotropic transition-accessibility rather then diffusion covariance. }$m(x_i)$ captures the net preferred displacement and therefore provides directionality whereas $C(x_i)$ captures the direction on the data manifold with high transition probabilities. To evaluate the metric along continuous paths, we construct smooth extensions of $m$ and $C$, for example using kernel smoothing, neural networks, or support vector regression \citep{smolaTutorialSupportVector2004}, while preserving positive semidefiniteness of \(C\). We use kernel smoothing for simplicity details are provided in \ref{app:Kernel smoothing}.
To obtain a non-degenerate Finsler structure, we define the regularized Freidlin–Wentzell (FW) action \citep{freidlinRandomPerturbationsDynamical2012a}. Firstly, set
$C_{\rho}(x)\coloneqq C(x)+\rho I$ with $\rho>0$ and define
\begin{equation}
\label{eq:empirical-fw-metric}
 F_{P}(x,v)
\coloneqq
\underbrace{\sqrt{v^\top C_{\rho}(x)^{-1}v}}_{\text{Anisotropic accessibility}}\,
\underbrace{\sqrt{
 m(x)^\top C_{\rho}(x)^{-1} m(x)
+\lambda^2}}_{\text{first-moment calibration}}
-
\underbrace{v^\top C_{\rho}(x)^{-1}m(x)}_{\text{directional alignment}},
\end{equation}
where $\lambda>0$ ensures Randers admissibility \ref{lem: Finsler admissibility}. The anisotropic accessibility term encourages geodesics to remain close to regions and directions that are well supported by the transition probabilities, directional alignment ensures that geodesics follow the preferred directionality and the first-moment calibration ensures that both penalty terms are on the same scale. The corresponding action is
\begin{equation}
\label{eq:empirical-fw-action}
S(\gamma)
=
\int_0^1
 F_{P}\!\left(\gamma(t),\dot\gamma(t)\right)
\,\mathrm dt .
\end{equation}
The following proposition formalizes how geodesics under \(F_P\) align with the transition structure encoded by \(P\):
\begin{tcolorbox}[
    enhanced,
    breakable,
    colback=black!2,
    colframe=black!35,
    boxrule=0.5pt,
    arc=1.5mm,
    left=2mm,
    right=2mm,
    top=1.5mm,
    bottom=1.5mm
]
\begin{proposition}[Informal \ref{prop:transition-geodesics}]
Suppose the endpoints can be connected along a transition-supported corridor at a Finsler cost lower than that required to leave and return to a prescribed surrounding neighborhood. Then every minimizing $F_P$-geodesic remains within that neighborhood. Moreover, locally the metric favors
alignment with the first transition moment.
\end{proposition}
\end{tcolorbox}
In the special case where $P$ is the discrete observation of an SDE \ref{def:nodewise-consistency}, we can conclude even more about geodesics. By  \citet{freidlinRandomPerturbationsDynamical2012a} and \citet{heymannGeometricMinimumAction2008a} we can show that geodesics under the resulting metric $F_P$ follow the maximum likelihood paths of the stochastic process.

\begin{tcolorbox}[
    enhanced,
    breakable,
    colback=black!2,
    colframe=black!35,
    boxrule=0.5pt,
    arc=1.5mm,
    left=2mm,
    right=2mm,
    top=1.5mm,
    bottom=1.5mm
]
\begin{theorem}[Informal \ref{thm:empirical-gmam}]
\label{thm}
If $P$ is a discretization of a SDE, then minimizers of the empirical FW Finsler action converge to the asymptotically most probable transition path in the small-noise limit.
\end{theorem}
\end{tcolorbox}
FFM is defined for arbitrary row-stochastic transition matrices, whereas its diffusion interpretation requires
Assumption~\ref{ass:generator-consistency}.
Connections to the limits of local graph kernels, and the scope
of these results for velocity-based constructions, are discussed in \ref{app:Generator consistency}.
\paragraph{Learning Geodesics:}
For endpoints $(x_0, x_1)\sim p_0 \otimes p_1$ and $t \in [0,1]$, parameterize a candidate curve by
\begin{equation}\label{eq:interpolant}
x_{t,\eta}(x_0, x_1) \;=\; (1-t)\,x_0 + t\,x_1 + t(1-t)\, \varphi_\eta(t, x_0, x_1),
\end{equation}
where $\varphi_\eta : [0,1] \times \R^d \times \R^d \to \R^d$ is some neural network with parameters $\eta$. The boundary correction $t(1-t)$ enforces $x_{0,\eta} = x_0$ and $x_{1,\eta} = x_1$. The time derivative is given by,
\begin{align*} \dot x_{t, \eta}= x_1-x_0
+
(1-2t)\varphi_\eta(t,x_0,x_1)
+
t(1-t)\partial_t\varphi_\eta(t,x_0,x_1).
\end{align*}
This interpolant was previously proposed in \citet{kapusniakMetricFlowMatchinga}, and similarly to their approach, we minimize the Finsler Energy functional,
\begin{equation}
\label{eq:geodesic_loss}
\mathcal{L}_{\mathrm{geo}}(\eta)
=
\mathbb{E}_{\substack{
(x_0,x_1)\sim p_0\otimes p_1\\
t\sim\mathcal{U}[0,1]
}}
\left[
F_P\!\left(
x_{t,\eta}(x_0,x_1),
\dot{x}_{t,\eta}(x_0,x_1)
\right)^2
\right].
    \end{equation}
Since energy minimizers are constant-speed
reparameterizations of geodesics, minimizing
\eqref{eq:geodesic_loss} over $\eta$ yields an approximation of the
pairwise $P$-derived Finsler geodesics. We denote the resulting parameters by
$\eta^*$.
\begin{algorithm}[t]
\caption{Training the Finsler Geodesic Model}
\label{alg:geodesic-training}
\begin{algorithmic}[1]
\Require Source and target distributions $p_0,p_1$, Finsler metric $F_P$,
         interpolant network $\varphi_\eta$, LR $\gamma$.
\While{not converged}
    \State Sample $x_0\sim p_0$, $x_1\sim p_1$, $t\sim\mathcal{U}[0,1]$ and compute $x_{t,\eta}(x_0,x_1)$ and $\dot{x}_{t,\eta}(x_0,x_1)$
    \State $\mathcal{L}_{\mathrm{geo}}
    \gets F_P\!\left(x_{t,\eta},\dot{x}_{t,\eta}\right)^2$
    \State $\eta \gets \eta-\gamma\nabla_\eta\mathcal{L}_{\mathrm{geo}}$
\EndWhile
\end{algorithmic}
\end{algorithm}
\paragraph{Finsler Optimal Transport:}
Having learned an approximation of the geodesic connecting arbitrary endpoint
pairs, we can approximate the inner minimization in the transport problem \ref{eq:OT Finsler Problem}.
For each $(x_0,x_1)$, define the learned Finsler transport cost
\begin{equation*}
\label{eq:learned_finsler_cost}
c_{\eta^*}(x_0,x_1)
:=
\int_0^1
F_P\!\left(
x_{t,\eta^*}(x_0,x_1),
\dot{x}_{t,\eta^*}(x_0,x_1)
\right)^2
\,\mathrm{d}t.
\end{equation*}
Thus,
\begin{equation*}
\label{eq:approx_inner_geodesic}
c_{\eta^*}(x_0,x_1)
\approx
\inf_{\substack{
\gamma(0)=x_0\\
\gamma(1)=x_1
}}
\int_0^1
F_P\!\left(
\gamma(t),\dot{\gamma}(t)
\right)^2
\,\mathrm{d}t  = c_F^2(x_0, x_1).
\end{equation*}
We can solve the corresponding static optimal transport problem using costs. In particular, we consider its
entropy-regularized form
\begin{equation}
\label{eq:entropic_finsler_ot}
\pi_{\varepsilon}^{*}
=
\arg\min_{\pi\in\Pi(\hat p_0,\hat p_1)}
\left\{
\sum_{i=1}^{m}\sum_{j=1}^{m}
\pi_{ij} c_{ij}
+
\varepsilon\,
\mathrm{KL}\!\left(
\pi\,\middle\|\,\hat p_0\otimes\hat p_1
\right)
\right\},
\end{equation}
where $\varepsilon>0$ controls the strength of the entropic regularization.
This problem can be solved efficiently using Sinkhorn iterations \citep{cuturiSinkhornDistancesLightspeed2013a}.
The interpolant $x_{t,\eta^*}$ and the coupling $\pi_{\varepsilon}^{*}$ provide an approximation of the
original Finsler OT problem \ref{eq:OT Finsler Problem} up to the approximation errors.
\paragraph{Flow and Stochastic Model:} Assume we have computed $\left(\pi_{\varepsilon}^{*},x_{\cdot,\eta^*}\right)$.
Drawing $(x_0,x_1)\sim\pi_{\varepsilon}^{*}$ and transporting each pair along its
learned geodesic gives the probability path $p_t(x)
\coloneqq\left(x_{t,\eta^*}\right)_{\#}\pi_{\varepsilon}^{*}(x)$. By construction, $p_0$ and $p_1$ are the endpoint marginals. The corresponding conditional probability distribution path is $p_t(x|x_0,x_1)=\delta_{x_{t,\eta}(x_0,x_1)}$. Furthermore,  define the conditional path 
$\mu_t(x_0,x_1):= x_{t,\eta^\star}(x_0,x_1)$
We then train a model $v_\theta$ using the conditional flow matching objective
\begin{align*}
\label{eq:finsler_flow_matching_loss}
\mathcal{L}_{\mathrm{CFM}}(\theta)
=
\mathbb{E}_{\substack{
(x_0,x_1)\sim\pi_{\varepsilon}^{\star}\\
t\sim\mathcal{U}[0,1]
}}
\left[
\left\|
v_\theta\!\left(
t,\mu_t(x_0,x_1)
\right)
-
\dot \mu_t(x_0,x_1)
\right\|_2^2
\right].
\end{align*}
Motivated by \citet{follmerRandomFieldsDiffusion1988a} we approximate each conditional bridge with an anisotropic Gaussian path. For $(x_0,x_1)\sim\pi^\star_\varepsilon$, define $C_{\rho,t} := C_\rho(\mu_t)$ and
$p_t(x\mid x_0,x_1)=\mathcal{N}\!\left(x;\,\mu_t,\,\sigma^2 t(1-t) C_{\rho,t} \right).$ We evaluate $C$ at the deterministic interpolant $\mu_t$ rather than at the state $x$, so that the conditional path can be sampled in closed form. The interpretation of this path is that noise is injected along accessible directions according to $P$. The conditional probability-flow vector field is
\begin{align*}
    u_t(x\mid x_0,x_1)
    =
    \dot{\mu}_t
    +
    \frac{1-2t}{2t(1-t)}(x-\mu_t)
    +
    \frac{1}{2}
    \dot{C}_{\rho,t}
    C_{\rho,t}^{-1}
    (x-\mu_t),
\end{align*}
See \ref{prop: vectorfield calc} for a derivation. We distill the conditional bridges into a marginal vector field
$v_\theta$ and score $s_\psi$ using simulation-free conditional flow and
score matching \citep{tongSimulationFreeSchrodingerBridges2024a},
\begin{align*}
\mathcal L_{\mathrm{CFM}}(\theta)
&=
\mathbb E_{\substack{
(x_0,x_1)\sim\pi_\varepsilon^\star,\;
t\sim\mathcal U(0,1)\\
x\sim p_t(\cdot\mid x_0,x_1)
}}
\left[
\left\|
\tilde{v}_\theta(t,x)
-
u_t(x\mid x_0,x_1)
\right\|_2^2
\right],
\\
\mathcal L_{\mathrm{CSM}}(\psi)
&=
\mathbb E_{\substack{
(x_0,x_1)\sim\pi_\varepsilon^\star,\;
t\sim\mathcal U(0,1)\\
x\sim p_t(\cdot\mid x_0,x_1)
}}
\left[
\lambda(t)^2
\left\|
s_\psi(t,x)
-
s_t(x\mid x_0,x_1)
\right\|_2^2
\right]
\end{align*}
where the final loss is given by $\mathcal L_{\mathrm{[SF]}^2\mathrm{M}}=\mathcal L_{\mathrm{CFM}}+\mathcal L_{\mathrm{CSM}}$ and the conditional score is 
\[s_t(x\mid x_0,x_1)=\nabla_x \log p_t(x\mid x_0,x_1)=-\frac{1}{\sigma^2 t(1-t)}C_{\rho,t}^{-1}(x-\mu_t).\]
\paragraph{Inference:} We sample \(x_0\sim p_0\) and integrate the learned dynamics from \(t=0\) to \(t=1\). In the deterministic setting, we integrate the ODE
$$ dX_t=v_\theta(t,X_t)\,dt. $$
In the stochastic setting, we use the score-and-flow SDE introduced in \ref{eq: SDE prelim} and make a slight generalization by setting
 $\mathbf g(t,x)=g(t)\, C_{\rho}^{1/2}(x), $ for some scalar-valued noise function $g(t)$. Substituting this choice into the stochastic dynamics gives
\begin{align*}
    dX_t=\left[\tilde{v}_\theta(t,X_t)+\frac{1}{2}g(t)^2\, C_{\rho}(x)s_\psi(t,X_t) + \frac{1}{2}g(t)^2\di (C_{\rho}(x))\right]dt+g(t)\, C_{\rho}(x)^{1/2}\,dW_t,
\end{align*}
where $\di (C_{\rho}(x))_i = \sum_{j}^d \partial_{j}C_{ij}(x)$ and which we integrate from \(X_0\sim p_0\) to \(t=1\). The scalar schedule \(g(t)\) controls the overall noise strength, whereas \(C_{\rho}(x)\) determines its local anisotropic geometry.

%% file: sections/experiments.tex
\section{Experiments}
\paragraph{Experimental Setup:} The transition graph and derived geometric quantities are recomputed after removing the evaluation cells. Our primary metric is Euclidean \(W_2\) at \(t=1/2\), supplemented by endpoint \(W_2\) and the geometry-dependent \(W_{2,F}\). Because $t=1/2$ is a
benchmark convention rather than an established correspondence with biological midpoint time, on Pancreas we additionally report $\overline W_2:= \frac{1}{|\mathcal T|}\sum_{t\in\mathcal T}
 W_2(p_t,p_{\mathrm{withheld}})$. We compare against OT-CFM, MFM/LAND, Curly-FM under matched architectures and training budgets. More details \cref{appendix C}. 
\paragraph{FFM recovers directed and anisotropic transition structure:} To test whether FFM can recover intermediate dynamics when both directionality matters and when the data satisfy the manifold hypothesis, we construct the synthetic Sheet benchmark, a saddle-shaped manifold. We score against a narrow corridor, which represents a subset of the manifold corresponding to high probability transition regions. We define a directed transition kernel that drives motion towards that middle corridor before crossing and away from it afterward. Curly-FM receives the first transition moment as its local velocity field, whereas FFM additionally uses the second moment to capture anisotropic accessibility. FFM substantially improves midpoint recovery, achieving \(W_2 = 0.1588\) and \(0.1449\)  for deterministic and stochastic FFM, respectively, compared with \(0.5221\) for Curly-FM \cref{fig:sheet-paths}. We further performed and ablation of the Finsler metric, which demonstrates that both the directional and anisotropic components contribute to this improvement \cref{fig:sheet_appendix_ablation}. On a complementary synthetic bifurcation benchmark, FFM also preserves cell-level fate assignments across both branches \cref{fig:forksheet-appendix}.
\begin{figure}[t]
\centering
\resizebox{0.5\linewidth}{!}{%
\begin{tabular}{lrrr}
\toprule
\cmidrule(lr){2-4}
Method
& mid $W_2\downarrow$
& end $W_2\downarrow$
& mid $W_{2,F}\downarrow$ \\
\midrule
OT-CFM
& 0.7112 $\pm$ 0.0233
& 0.1734 $\pm$ 0.0216
& 1.0305 $\pm$ 0.0263 \\
MFM (LAND)
& 0.6843 $\pm$ 0.0074
& 0.2357 $\pm$ 0.0258
& 0.8645 $\pm$ 0.0961 \\
Curly-FM
& 0.5221 $\pm$ 0.0016
& \textbf{0.1405 $\pm$ 0.0057}
& 0.5236 $\pm$ 0.1657 \\
FFM (ours)
& 0.1588 $\pm$ 0.0096
& 0.2343 $\pm$ 0.0292
& 0.2892 $\pm$ 0.0090 \\
FFM (ours), with noise
& \textbf{0.1449 $\pm$ 0.0116}
& 0.2636 $\pm$ 0.0682
& \textbf{0.2742 $\pm$ 0.0314} \\
\midrule
\emph{sampling floor}
& \emph{0.0646}
& \emph{0.1896}
& \emph{0.2812} \\
\bottomrule
\end{tabular}
}

\includegraphics[width=0.75\linewidth]{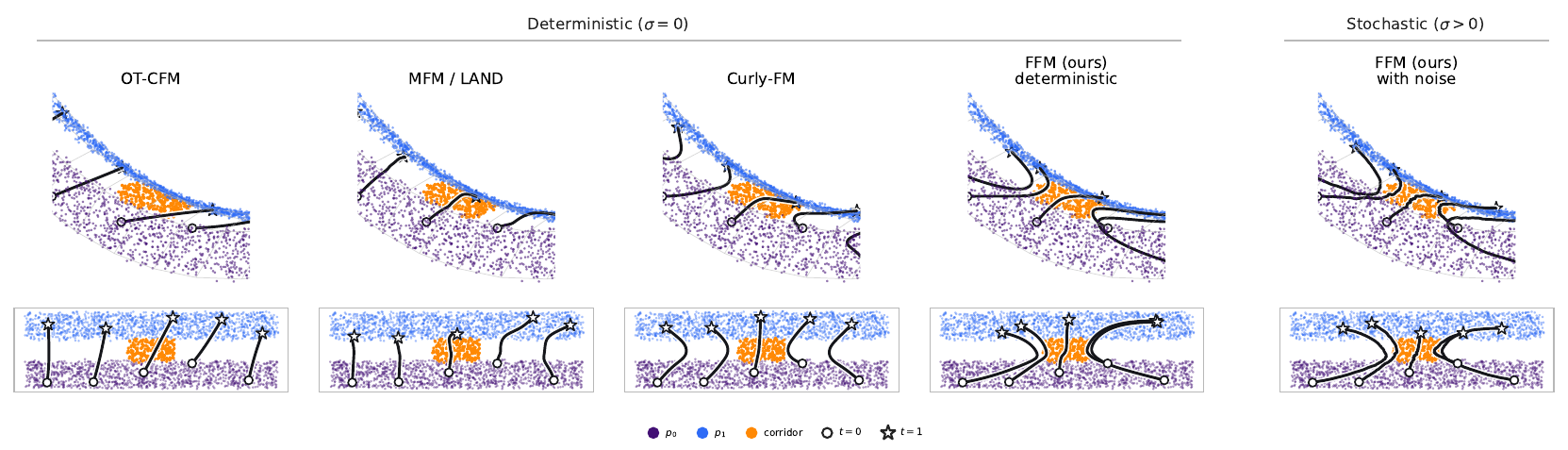}

\vspace{-0.4em}

\caption{
$W_{2,F}$ denotes the Wasserstein distance induced by the Finsler cost $c_F$.
Sampling floor is computed between disjoint halves of the corresponding
reference cloud.
}
\label{fig:sheet-paths}
\end{figure}
\paragraph{FFM improves recovery of pancreatic differentiation in high-dimensional representations:} To test whether FFM's graph-derived geometry improves trajectory reconstruction on real RNA-velocity data, we evaluate FFM on 3,696 pancreatic cells \citep{bastidas-ponceComprehensiveSingleCell2019a} partitioned into latent-time tertiles and withhold the intermediate population, which is enriched for Ngn3-high pre-endocrine cells.  FFM receives an RNA-velocity-derived transition kernel computed using \citet{langeCellRankDirectedSinglecell2022a}, whereas Curly-FM receives per-cell velocities. We found that relative performance depended strongly on representation dimensionality. In two dimensions, Curly-FM achieved lower midpoint $W_2$ than FFM, whereas FFM achieved substantially lower midpoint $W_2$ in the $20$ and $50$dimensional PCA \cref{fig:paths_and_quantitative_results}. endpoint rankings vary across methods \cref{fig:paths_and_quantitative_results}
\begin{figure*}[t]
\centering
\begin{minipage}[c]{0.35\textwidth}
    \centering
    \includegraphics[width=\linewidth]{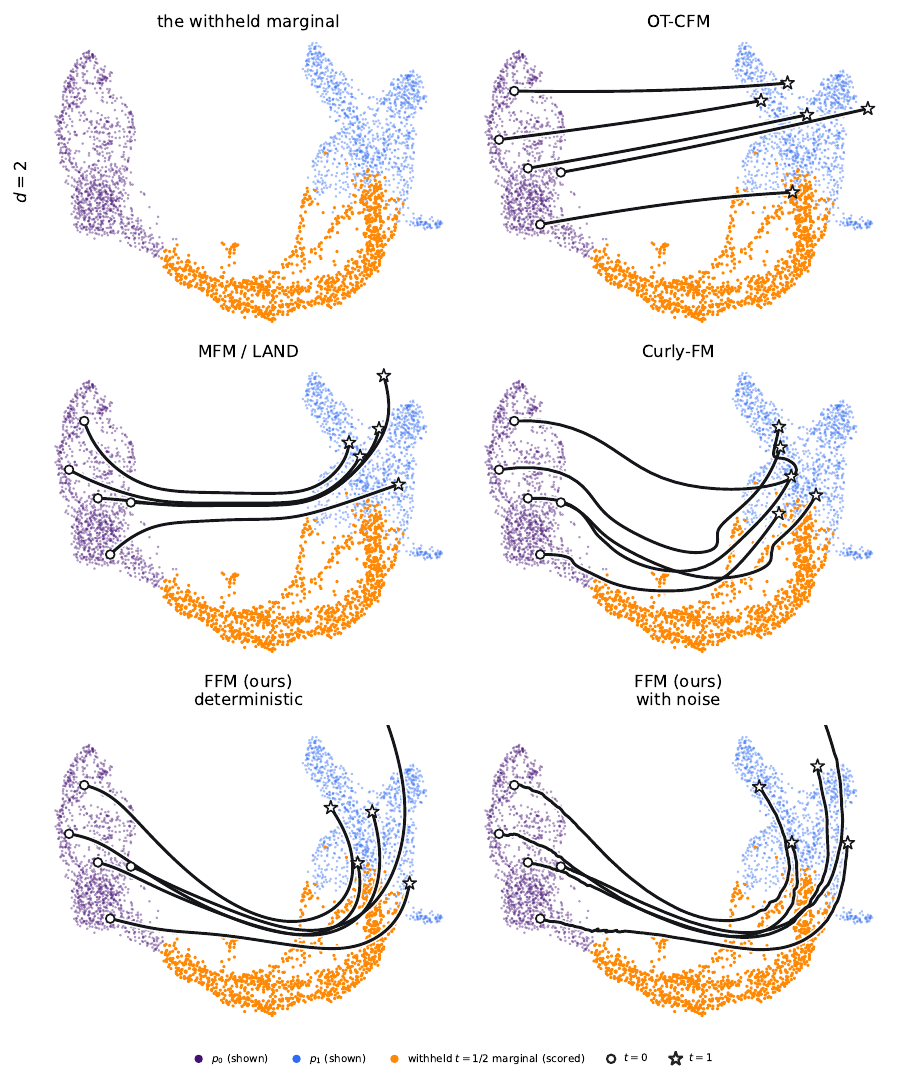}
\end{minipage}
\hfill
\begin{minipage}[c]{0.61\textwidth}
    \centering
    \scriptsize
    \setlength{\tabcolsep}{2.6pt}
    \renewcommand{\arraystretch}{0.98}

    \resizebox{\linewidth}{!}{%
    \begin{tabular}{lrrrr}
\toprule
Method
& $W_2(\frac12)\downarrow$
& $\overline W_2^{[1/3,2/3]}\downarrow$
& $W_2(1)\downarrow$
& $W_{2,F}(\frac12)\downarrow$ \\
\midrule
\multicolumn{5}{l}{\textit{$d=2$}} \\
OT-CFM
& 1.8062 $\pm$ 0.0197
& 1.7869 $\pm$ 0.0230
& \textbf{0.1740 $\pm$ 0.0338}
& 0.6420 $\pm$ 0.0059 \\
MFM / LAND
& 1.4682 $\pm$ 0.0128
& 1.4642 $\pm$ 0.0131
& 0.3769 $\pm$ 0.0346
& \textbf{0.5399 $\pm$ 0.0024} \\
Curly-FM
& \textbf{0.9773 $\pm$ 0.1311}
& \textbf{0.9686 $\pm$ 0.1725}
& 0.2581 $\pm$ 0.0371
& 0.7493 $\pm$ 0.0258 \\
FFM (ours), deterministic
& 1.6692 $\pm$ 0.0154
& 1.6175 $\pm$ 0.0348
& 0.3283 $\pm$ 0.1293
& 0.5681 $\pm$ 0.0100 \\
FFM (ours), with noise
& 1.6957 $\pm$ 0.0305
& 1.6434 $\pm$ 0.0555
& 0.2272 $\pm$ 0.0239
& 0.5730 $\pm$ 0.0082 \\
\emph{sampling floor}
& \emph{0.1113}
& \emph{0.1113}
& \emph{0.1022}
& \emph{0.4162} \\
\midrule
\multicolumn{5}{l}{\textit{$d=20$}} \\
OT-CFM
& 12.3484 $\pm$ 0.0485
& 12.4957 $\pm$ 0.0681
& 5.2262 $\pm$ 0.1583
& 10.9277 $\pm$ 0.0412 \\
MFM / LAND
& 12.0705 $\pm$ 0.0095
& 12.1656 $\pm$ 0.0522
& 7.8205 $\pm$ 0.7863
& 10.4768 $\pm$ 0.0637 \\
Curly-FM
& 12.2962 $\pm$ 0.0414
& 12.3920 $\pm$ 0.0687
& \textbf{4.8562 $\pm$ 0.0576}
& 11.0235 $\pm$ 0.0520 \\
FFM (ours), deterministic
& 10.4319 $\pm$ 0.0896
& 10.7413 $\pm$ 0.1927
& 5.3793 $\pm$ 0.1848
& \textbf{8.9033 $\pm$ 0.0400} \\
FFM (ours), with noise
& \textbf{10.4036 $\pm$ 0.0279}
& \textbf{10.7355 $\pm$ 0.0560}
& 5.0605 $\pm$ 0.1654
& 8.9213 $\pm$ 0.0453 \\
\emph{sampling floor}
& \emph{4.1781}
& \emph{4.1781}
& \emph{4.4483}
& \emph{4.3608} \\
\midrule
\multicolumn{5}{l}{\textit{$d=50$}} \\
OT-CFM
& 13.5145 $\pm$ 0.0466
& 13.6811 $\pm$ 0.1314
& 8.1252 $\pm$ 0.1103
& 15.5900 $\pm$ 0.2287 \\
MFM / LAND
& 12.9526 $\pm$ 0.0605
& 13.0996 $\pm$ 0.1029
& 8.6687 $\pm$ 0.3880
& 14.6185 $\pm$ 0.1085 \\
Curly-FM
& 13.3381 $\pm$ 0.1482
& 13.4440 $\pm$ 0.2069
& \textbf{7.2931 $\pm$ 0.2558}
& 15.0427 $\pm$ 0.1562 \\
FFM (ours), deterministic
& 12.5784 $\pm$ 0.0235
& \textbf{12.7715 $\pm$ 0.0719}
& 7.7975 $\pm$ 0.1005
& 13.9368 $\pm$ 0.1217 \\
FFM (ours), with noise
& \textbf{12.5732 $\pm$ 0.0573}
& 12.7723 $\pm$ 0.0952
& 7.7411 $\pm$ 0.1297
& \textbf{13.7800 $\pm$ 0.0861} \\
\emph{sampling floor}
& \emph{6.7529}
& \emph{6.7529}
& \emph{7.0275}
& \emph{7.2123} \\
\bottomrule
\end{tabular}
}

\vspace{1mm}

\resizebox{\linewidth}{!}{%
\begin{tabular}{@{}lcc@{}}
\toprule
Method & Shared coupling & Method-specific coupling \\
\midrule
Graph-drift FM (control)
& 11.9409 $\pm$ 0.1110
& 11.9223 $\pm$ 0.0909 \\
Anisotropic fixed-time FM \cref{fixed-time FW}
& 11.5712 $\pm$ 0.0880
& 11.4344 $\pm$ 0.1067 \\
Second-moment FFM ($m=0$)
& 11.5424 $\pm$ 0.1226
& 11.4967 $\pm$ 0.0878 \\
FFM, one-form deleted (existing ablation)
& 11.1740 $\pm$ 0.0759
& 11.1183 $\pm$ 0.0774 \\
FFM (ours)
& \textbf{10.3967 $\pm$ 0.1371}
& \textbf{10.3960 $\pm$ 0.1396} \\
\bottomrule
\end{tabular}
    }
\end{minipage}

\caption{
Pancreas trajectory reconstruction. Left: Two-dimensional
representation. Top right: $W_2$ at $t=1/2$, the average
$\overline W_2$ over 11 fixed evaluation times in $[1/3,2/3]$, endpoint $W_2$, and $W_{2,F}$ Bottom right: same-information objective ablation at $d=20$.
``Shared'' fixes the endpoint coupling
``method-specific'' constructs the coupling separately for each method.
Ablation values are $W_2(t=1/2)$, mean $\pm$ population standard deviation
over five seeds.
}
\label{fig:paths_and_quantitative_results}
\end{figure*}

\paragraph{FFM integrates diverse sources of dynamical information through a common transition representation:}
 To demonstrate that FFM can integrate heterogeneous dynamical information through a common transition representation, we evaluate 2,901 lineage-recorded cells from the iTracer hindbrain organoid dataset \citep{heLineageRecordingHuman2022}, withholding the middle diffusion-pseudotime tertile. Lineage information is encoded by barcode families and finer-grained scar families. We construct transition kernels from expression alone and augment them with RNA velocity, lineage similarity derived from barcode and scar families, or both. Each FFM variant achieves lower mean midpoint \(W_2\) than OT-CFM and MFM across the evaluated representations \cref{tab:itracer_main}. We further found that the optimal kernel depends on the representation dimensionality. Velocity performs best in two dimensions, lineage plus velocity in 20 dimensions, and expression alone in 50 dimensions. These results show that FFM can incorporate heterogeneous sources of transition information, while the utility of each modality depends on how well it informs the transition kernel.
\begin{table}[t]
\noindent

\begin{minipage}[t]{0.70\linewidth}
\vspace{0pt}
\resizebox{\linewidth}{!}{%
\begin{tabular}{lccc}
\toprule
Method & $d=2$ mid $W_2$ & $d=20$ mid $W_2$ & $d=50$ mid $W_2$ \\
\midrule
OT-CFM
& $1.3631 \pm 0.0368$
& $15.9054 \pm 0.0804$
& $18.8807 \pm 0.1172$ \\
MFM / LAND
& $1.3563 \pm 0.0537$
& $15.0519 \pm 0.2250$
& $18.2036 \pm 0.2663$ \\
FFM, expression only
& $1.7162 \pm 0.9919$
& $12.8597 \pm 0.1007$
& $\mathbf{16.6979 \pm 0.3020}$ \\
FFM, + velocity
& $\mathbf{1.0204 \pm 0.1040}$
& $12.7288 \pm 0.1253$
& $16.7371 \pm 0.2566$ \\
FFM, + lineage
& $1.3074 \pm 0.1083$
& $12.7368 \pm 0.1455$
& $16.7415 \pm 0.2426$ \\
FFM, + lineage + velocity
& $1.1773 \pm 0.1404$
& $\mathbf{12.7121 \pm 0.1314}$
& $16.9196 \pm 0.2700$ \\
\bottomrule
\end{tabular}
}
\end{minipage}
\hfill
\begin{minipage}[t]{0.27\linewidth}
\vspace{0pt}
\captionsetup{
    justification=raggedright,
    singlelinecheck=false,
    skip=0pt
}
\caption{Intermediate-distribution recovery on iTracer R2 for different
transition-kernel constructions. Lower is better.}
\label{tab:itracer_main}
\end{minipage}

\end{table}
\paragraph{Additional analyses:} Appendix \cref{appendix C} 
reports runtime and parameter-sensitivity analyses, additional results on erythroid differentiation \citep{pijuan-salaSinglecellMolecularMap2019} as used in used in \citet{petrovicCurlyFlowMatching2025a}, and further details of the benchmarks above.

%% file: sections/discussion.tex
\section{Discussion}
We present Finsler Flow Matching (FFM), a framework for transforming discrete Markov transition graphs into continuous stochastic dynamics through a transition-derived Finsler geometry. Its joint first- and second-moment construction captures directional preference and anisotropic accessibility, which guide source--target coupling, geodesic interpolation, and stochastic conditional paths.
\\
Our experiments suggest that this is particularly useful for strongly directional or anisotropic dynamics. On the synthetic Sheet benchmark, ablating either component degrades intermediate-state recovery, while on Pancreas the advantage of FFM becomes more pronounced in higher-dimensional representations. The iTracer experiment further demonstrates the flexibility of the Markov transition kernel as an interface between heterogeneous measurements and generative dynamics. Expression, RNA velocity, lineage information, and their combinations can all be incorporated through the same representation. Nevertheless, additional modalities do not uniformly improve recovery. FFM therefore inherits both the information and the limitations of the supplied transition kernel, highlighting the importance of how experimental measurements are translated into transition probabilities.
\\
These findings should be interpreted within the scope of the supplied dynamical prior. FFM does not resolve the non-identifiability of cellular dynamics from snapshots alone, and interpreting its geodesics as most-probable transition curves requires the stated diffusion-consistency and small-noise assumptions. Moreover, the Gaussian conditional paths used for stochastic training are tractable approximations rather than exact bridges of the graph-implied reference diffusion. Future work could explicitly model uncertainty in the transition graph and develop richer conditional path families. Another promising direction is conditional FFM, in which experimental covariates or perturbations parameterize the induced geometry, allowing the model to capture condition-specific changes in trajectory dynamics.

%% file: sections/appendix.tex
\begin{appendix}
\section{Prelimiaries}
\paragraph{Finsler Geometry:}\label{appendix:Finsler} For a complete introduction to Finsler geometry we refer the reader to \citet{baoIntroductionRiemannFinslerGeometry2000a}.
Let $M$ be an $m-$dimensional manifold embedded in $\R^d$. The \emph{tangent bundle} $TM \coloneqq \bigsqcup_{x \in M} T_x M$ is the $2m$-dimensional manifold composed of pairs $(x,v)$ with $x \in M$ and $v \in T_x M$, where $T_xM$ is the \emph{tangent space} at $x$.
\begin{definition}[Finsler structure]\label{def:finsler}
 A \emph{Finsler structure} on $M$ is a continuous function $F: TM \to \R_{\geq 0}$, which is  smooth on $TM \setminus \{0\}$, such that for all $(x,v) \in TM$:
\begin{enumerate}
    \item[(F1)] \emph{Positive homogeneity:} $F(x, \lambda v) = \lambda F(x,v)$ for all $\lambda > 0$.
    \item[(F2)] \emph{Strong convexity:} the Hessian matrix $g_{ij}(x,v) = \tfrac{1}{2}\partial^2 F^2 / \partial v^i \partial v^j$ is positive definite for $v \neq 0$.
    \item[(F3)] \emph{Positivity:} $F(x,v) > 0$ for $v \neq 0$.
\end{enumerate}
\end{definition}
In general, we have that $F(x,-v) \neq F(x,v)$ hence, the cost of moving in direction $v$ is different from the cost of moving in direction $-v$. A Riemannian metric is the special case $F(x,v) = \sqrt{v^\top G(x) v}$,  which additionally satisfies $F(x,-v) = F(x,v)$.
Analogously to Riemannian Geometry, one can define a notion of length and geodesics.
\begin{definition}[Length and geodesics]
For an absolutely continuous curve $\gamma: [0,1] \to M$, define
\[
L_F(\gamma) = \int_0^1 F(\gamma(t), \dot\gamma(t)) \,\dd t, \label{eq:length functional} \qquad E_F(\gamma) = \int_0^1 F(\gamma(t), \dot\gamma(t))^2 \,\dd t.
\]
A \emph{geodesic} between $x_0, x_1 \in \R^d$ is a curve minimizing $L_F$ subject to $\gamma(0) = x_0$, $\gamma(1) = x_1$. Note that energy minimizers are constant-speed parametrized geodesics. Furthermore, we define the \emph{length metric }on $M$ as:
\[
c_F(x_0,x_1)
=
\inf_{\gamma}
\int_0^1 F(\gamma_t,\dot\gamma_t)\,dt,
\]
where the infimum is taken over curves satisfying $\gamma_0=x_0,\ \gamma_1=x_1$.
\end{definition}
We are particularly interested in the special case called a \textit{Randers Metric}, where $G(x)$ is a smooth Riemannian metric and  $\beta(x) $ a smooth $1$-form and we define
\begin{equation}\label{eq:randers}
F(x,v) = \sqrt{v^\top G(x) v} \;+\; \beta(x)^\top v.
\end{equation}
\begin{lemma} \label{lem: Finsler admissibility}[\citet{baoIntroductionRiemannFinslerGeometry2000a}]
The Randers structure \eqref{eq:randers} is a Finsler structure if and only if $\norm{\beta(x)}_{G\inv(x)} < 1$ pointwise, where $\norm{u}_{G\inv}^2 = u^\top G\inv u$.
\end{lemma}
\begin{remark}[Approximation of the data manifold]
Although the underlying dynamics may be concentrated on an unknown
low-dimensional manifold $\mathcal M\subset\mathbb R^d$, we do not
estimate $\mathcal M$. Instead, we construct the Finsler
metric on $\mathbb{R}^d$ so that motion along directions supported
by the local data and transition geometry has low cost, whereas motion
away from the data manifold, in particular along weakly supported
normal directions, has high cost. Thus, the Finsler structures in this work will be defined on $\R^d$.
\end{remark}
\begin{definition}[$P$-derived Randers metric]\label{def:Pderived}
A $P$-derived Randers metric is a Randers metric
\[
F_P(x_i,v)
=
\sqrt{v^\top G_P(x_i)v}
+
\beta_P(x_i)^\top v
\]
whose metric tensor $G_P(x_i)$ and one-form $\beta_P(x_i)$ are
constructed from local transition statistics of $P$, with
\[
G_P(x_i)\in\mathrm{SPD}(d),
\qquad
\|\beta_P(x_i)\|_{G_P(x_i)^{-1}}<1.
\]
\end{definition}
For example, if $C_P(x_i)$ denotes a local transition second-moment
tensor, one may consider the family
\[
G_{P,q}(x_i)
=
\bigl(\tau I+C_P(x_i)\bigr)^{-q},
\qquad q>0,
\]
together with any admissible transition-derived one-form $\beta_P$.
For $q=1$, this is closely related to local-kernel geometries, where
the second moment determines the induced Riemannian metric
\citep{berryLocalKernelsGeometric2015a}.
\paragraph{Optimal Transport:} Optimal transport (OT) was originally formulated by Monge as the problem of finding a deterministic transport map $T:\mathbb{R}^d\to\mathbb{R}^d$ pushing $p_0$ onto $p_1$ while minimizing the transport cost:
\[
\inf_{T:\,T_{\#}p_0=p_1}
\int_{\mathbb{R}^d} c(x,T(x))\,\mathrm{d}p_0(x).
\]
Since a transport map cannot split mass, such a minimizers does not exist in general. For example, if
\[
p_0=\delta_0,
\qquad
p_1=\frac12\delta_{-1}+\frac12\delta_{1},
\]
no map $T$ can satisfy $T_{\#}p_0=p_1$. The Kantorovich relaxation resolves this issue by optimizing instead over transport plans $\pi\in\Pi(p_0,p_1)$:
\[
\inf_{\pi\in\Pi(p_0,p_1)}
\int_{\mathbb{R}^d\times\mathbb{R}^d}
c(x_0,x_1)\,\mathrm{d}\pi(x_0,x_1).
\]
In the example above, an admissible plan would be
\[
\pi=\frac12\delta_{(0,-1)}+\frac12\delta_{(0,1)}.
\]
Under the Euclidean squared distance $c(x_0, x_1) = \frac{1}{2}\Vert{}x_0 - x_1\Vert{}^2$ (absolute continuity of $p_0$) \emph{Brenier's theorem} guarantees that the optimal plan $\pi^\star$ exists and is unique. Furthermore, there is a unique optimal transport map $T = \nabla \varphi$, for some convex function $\varphi$. The classical \emph{Benamou--Brenier theorem} establishes that this static transport problem is equivalent to a dynamic formulation over density paths $\rho_t$ and velocity fields $v_t$:
\begin{equation*}\label{eq:benamou-brenier-euclidean}
    W_2^2(p_0, p_1) = \inf_{(\rho, v)} \int_0^1 \int_{\mathbb{R}^d} |v_t(x)|^2  \mathrm{d}\rho_t(x) , \mathrm{d}t \quad \text{s.t.} \quad \partial_t \rho_t + \nabla \cdot (\rho_t v_t) = 0, \quad \rho_0 = p_0, \quad \rho_1 = p_1.
\end{equation*}
When substituting the Euclidean cost with the Finsler geodesic action cost $c_F(x_0, x_1)$ defined in \eqref{eq:finsler-geodesic-cost}, the static Kantorovich problem remains uniquely attained \citep[Theorem~4.10]{ohtaFinslerInterpolationInequalities2009a}. Moreover, the static problem coincides with the dynamic Finsler Benamou-Brenier formulation:
\begin{align*}
W_{2,F}^2(p_0,p_1) \coloneqq \inf_{\pi \in \Pi(p_0, p_1)} \int_{\mathbb{R}^d \times \mathbb{R}^d} c_F(x_0, x_1)^2  \mathrm{d}\pi(x_0, x_1) = \inf_{(\rho, v)} \int_0^1 \int_{\mathbb{R}^d} F(x, v_t(x))^2\mathrm{d}\rho_t(x)  \mathrm{d}t
\end{align*}
subject to the continuity equation $\partial_t \rho_t + \nabla \cdot (\rho_t v_t) = 0$ with endpoint conditions $\rho_0 = p_0$ and $\rho_1 = p_1$ \citep[Theorem~7.3 and Corollary~7.5]{ohtaHeatFlowFinsler2009a}.
\paragraph{Flow Matching:} We give a brief introduction to Flow Matching for a more in depth treatment consider \citet{holderriethIntroductionFlowMatchinga}. The idea behind flow matching is to find a vector field that transports one distribution into another and can generate samples by integrating said vector field. More formally one can think of it as  trying to find a  vector field
$u_t : [0,1) \times \mathbb{R}^n \rightarrow \mathbb{R}^n$ such that
\begin{align*}
\left\{
\begin{aligned}
&\frac{d}{dt}x_t = u_t(x_t), \\
&x_0 \sim p_0, x_1 \sim p_1 .
\end{aligned}
\right.
\end{align*}
Once such a vector field has been learned, samples from $p_1$ can be generated by sampling $x_0 \sim p_0$ and integrating the vector field from $t=0$ to $t=1$. In general, this vector field is not unique, which gives us some choice. For example, if \(p_0=p_1=p\), then any vector field satisfying $\nabla\cdot(pu)=0$ leaves the distribution \(p\) invariant.
More generally, the endpoint distributions \(p_0\) and \(p_1\) do not uniquely determine the intermediate probability path. We exploit this freedom by choosing a probability density path \((p_t)_{t\in[0,1]}\) connecting \(p_0\) and \(p_1\), and then finding a vector field whose flow realizes this path. Even for a fixed path \(p_t\), such a vector field need not be unique (one could make the choice optimal by choosing $u_t$ as the minimizer of some action).
\newline
\newline
 In general, we do not have $p_0$ and $p_1$ available in closed form, and hence the corresponding marginal path $p_t$ is intractable. The idea is to condition the path on samples; then the marginal distributions become Dirac distributions. We therefore introduce a coupling $\pi \in \Pi(p_0,p_1)$, i.e., a probability distribution on $\mathbb{R}^n \times \mathbb{R}^n$ with marginals $p_0$ and $p_1$, and define a tractable conditional probability path $p_t(x \mid x_0,x_1)$ with marginals $\delta_{x_0}$ and  $\delta_{x_1}$. The corresponding marginal path is
\begin{align*}
p_t(x)
=
\int_{\mathbb{R}^n \times \mathbb{R}^n}
p_t(x \mid x_0,x_1)\,d\pi(x_0,x_1).
\end{align*}
For every endpoint pair $(x_0,x_1)$, we define a conditional vector field $u_t(x \mid x_0,x_1)$ that solves 
\begin{align*}
    \left\{
    \begin{aligned}
    &\frac{d}{dt}x_t = u_t(x_t\mid x_0, x_1), \\
    &x_t \sim p_t(\cdot \mid x_0, x_1) \text{ for all } t \in [0,1].
    \end{aligned}
    \right.
\end{align*}
 Surprisingly, although the marginal vector field is generally intractable, the Conditional Flow Matching objective
\begin{align*}
\mathcal{L}_{\mathrm{CFM}}(\theta)
=
\mathbb{E}_{\substack{
(x_0,x_1)\sim\pi,\,
t\sim\mathcal{U}[0,1],\\
x\sim p_t(\cdot\mid x_0,x_1)
}}
\left[
\left\|
v_\theta(t,x)
-
u_t(x\mid x_0,x_1)
\right\|_2^2
\right]
\end{align*}
has the same gradient with respect to $\theta$ as the corresponding unconditional Flow Matching objective. Hence, we can learn the marginal vector field using only samples from tractable conditional probability paths, resulting in a simulation-free training procedure.
\paragraph{Simulation free Score and Flow Matching:} Here we introduce some of the ideas of \citet{tongSimulationFreeSchrodingerBridges2024a} \footnote{Which the authors highly suggest reading in full.}.
Similarly to the previous section, we want to find a vector field $\tilde{u}_t$ that satisfies. 
\begin{align*}
\left\{
\begin{aligned}
&\dd x_t = \tilde{u}_t(x_t) \dd t + g_t dW_t, \\
&x_0 \sim p_0, x_1 \sim p_1 .
\end{aligned}
\right.
\end{align*}
where $g_t$ is some scalar valued function controlling the diffusion strength and $W_t$ is a standard Brownian motion, i.e.\ a continuous stochastic process with independent Gaussian increments. Hence, instead of an ODE we treat a stochastic differential equation (SDE). 
\newline
\newline
In a first step we go back again the setting of the previous section and choose a conditional probability path $p_t (x\mid x_0, x_1)$ and learn a conditional vector field such that the corresponding ODE
\begin{align*}
\frac{d}{dt}x_t
=
u_t(x_t)
\end{align*}
satisfies $x_t \sim p_t$ for every $t$. We additionally introduce the conditional score of the probability path,
\begin{align*}
s_t(x \mid x_0, x_1)
:=
\nabla \log p_t(x \mid x_0, x_1 ).
\end{align*}
Note that similarly to flow matching, also the conditional score matching loss has the same gradient unconditional one. Using this score, the deterministic flow can be transformed into a stochastic process,
\begin{align*}
\dd x_t
=
\left[
v_t(x_t)
+
\frac{1}{2}g_t g_t^\top s_t(x_t)
\right]\dd t
+
g_t\,dW_t.
\end{align*}
The Fokker-Planck equation of this SDE is given by
\begin{align*}
\partial_t p_t
&=
-\nabla\cdot
\left(
p_t
\left[
v_t+\frac{1}{2}g_tg_t^\top s_t
\right]
\right)
+
\frac{1}{2}
\nabla\cdot
\left(
g_tg_t^\top \nabla p_t
\right) \\
&=
-\nabla\cdot(p_t v_t),
\end{align*}
where we used $p_t s_t=\nabla p_t$. Hence, the ODE and the corresponding SDE satisfy the same Fokker--Planck equation and therefore have the same marginals $X_t\sim p_t$. The score thus allows us to add stochasticity to the learned flow without changing the prescribed probability path.
\section{Theoretical Contributions}\label{sec: appendic Theoretical}
\paragraph{Transition-supported geodesics.}
For a threshold $\tau>0$, define the high-probability transition edges
\[
E_\tau
:=
\{(i,j):P_{ij}\geq\tau\}.
\]
For endpoints $x_0,x_1\in\mathcal D$, define the transition-supported
subset
\[
\mathcal S_\tau(x_0,x_1)
:=
\left\{
x_i\in\mathcal D:
x_i \text{ lies on a directed path from $x_0$ to $x_1$
using only edges in }E_\tau
\right\}.
\]
We assume that $\mathcal S_\tau(x_0,x_1)$ is non-empty. For $r>0$, define its $r$-neighborhood by
\[
U_r(\mathcal S_\tau)
:=
\left\{
x\in\mathbb R^d:
\operatorname{dist}(x,\mathcal S_\tau)<r
\right\},
\qquad
\operatorname{dist}(x,\mathcal S_\tau)
:=
\min_{y\in\mathcal S_\tau}\|x-y\|_2 .
\]
Furthermore, define
\[
c_r
:=
\min_{\substack{
r/2\leq\operatorname{dist}(x,\mathcal S_\tau)\leq r\\
\|v\|_2=1}}
F_P(x,v)
>0.
\]
The positivity follows since $F_P$ is continuous and strictly positive
for $v\neq0$, while the domain of minimization is compact. For our particular $P$-derived Randers metric \eqref{eq:empirical-fw-metric}, $c_r$ can be
related directly to the transition moments. By Cauchy--Schwarz,
\[
F_P(x,v)
\geq
\frac{
\sqrt{\|m(x)\|_{C_\rho(x)^{-1}}^2+\lambda^2}
-\|m(x)\|_{C_\rho(x)^{-1}}
}{
\sqrt{\lambda_{\max}(C_\rho(x))}}
\|v\|_2 .
\]
Hence, regions and directions with large $P$-weighted transition
covariance are relatively cheap, whereas regions with little transition
support, corresponding to small eigenvalues of $C_\rho$, incur a larger cost.
Thus annuli farther from $\mathcal S_\tau$ have a larger escape cost $c_r$.
\begin{proposition}[Confinement and directional preference of geodesics]
\label{prop:transition-geodesics}
Let $F_P$ be the Randers metric induced by the first and second
transition moments $m$ and $C$ in \eqref{eq:empirical-fw-metric}.
Suppose there exists a curve $\bar\gamma$ connecting $x_0$ and $x_1$
inside $U_{r/2}(\mathcal S_\tau)$. If
\[
L_{F_P}(\bar\gamma)<c_r r,
\]
then every minimizing geodesic connecting $x_0$ and $x_1$ is contained
in $U_r(\mathcal S_\tau)$.

Moreover, for fixed local speed
$\|v\|_{C_\rho(x)^{-1}}$, the Finsler cost is minimized when $v$ is
positively aligned with the first transition moment $m(x)$.
\end{proposition}

\begin{proof}
Suppose a minimizing geodesic $\gamma^\star$ leaves
$U_r(\mathcal S_\tau)$. Since $x_0,x_1\in\mathcal S_\tau$, it must cross
the annulus
\[
r/2\leq\operatorname{dist}(x,\mathcal S_\tau)\leq r
\]
once when leaving and once when returning. Since
$x\mapsto\operatorname{dist}(x,\mathcal S_\tau)$ is $1$-Lipschitz,
its total Euclidean length in this annulus is at least $r$.
By positive homogeneity and the definition of $c_r$,
\[
L_{F_P}(\gamma^\star)\geq c_r r
>
L_{F_P}(\bar\gamma),
\]
contradicting minimality. Hence
$\gamma^\star\subset U_r(\mathcal S_\tau)$.

For the directional claim, write $A(x)=C_\rho(x)^{-1}$. For fixed
$\|v\|_{A(x)}$, the first term of
\[
F_P(x,v)
=
\sqrt{\|m(x)\|_{A(x)}^2+\lambda^2}\,\|v\|_{A(x)}
-
\langle v,m(x)\rangle_{A(x)}
\]
is constant. Thus minimizing $F_P$ is equivalent to maximizing
$\langle v,m(x)\rangle_{A(x)}$, which by Cauchy--Schwarz is attained
when $v$ is a positive multiple of $m(x)$.
\end{proof}
\paragraph{Uniform Convergence of Empirical Drift and Diffusion:}
Let $Y_T$ be a stochastic process in $\R^d$ governed by the Stochastic Differential Equation (SDE)
\begin{align}
    dY_T = b_0(Y_T)\,dt + \Sigma_{\mathrm{sde}}(Y_T)\,dW_t,
    \label{eq:sde}
\end{align}
where $b_0\in C^\infty(\R^d,\R^d) $ is the drift vector, $W_t$ is a standard $d$-dimensional Brownian motion and $D(x) \coloneqq \Sigma_{\mathrm{sde}}(x)\Sigma_{\mathrm{sde}}(x)^\top$,
is the smooth diffusion matrix.
\begin{definition}[Transition Operator and Infinitesimal Generator]
The transition operator (semigroup) $T_t$ acting on a test function $f \in C_c^\infty(\mathbb{R}^d)$ is defined by$$T_t f(x) \coloneqq \mathbb{E}\left[f(Y_T) \mid Y_0 = x\right].$$The infinitesimal generator $L$ of the process $Y_T$ is the operator$$Lf(x) \coloneqq \lim_{t \to 0^+} \frac{T_t f(x) - f(x)}{t} = \nabla f(x)^\top b_0(x) + \frac{1}{2}\mathrm{Tr}\left(D(x) \nabla^2 f(x)\right).$$
\end{definition}
Suppose we observe samples $\{x_i\}_{i=1}^N \subset \mathbb{R}^d$ drawn i.i.d. from a continuous density $p(x)$ alongside a row-stochastic transition matrix $P \in \mathbb{R}^{N \times N}$ constructed at spatial scale $\varepsilon_N > 0$.Thus the discrete transition operator is
\[
(Pf)(x_i)=\sum_{j=1}^N P_{ij}f(x_j).
\]
\\
\\
To link discrete dataset observations to the continuous diffusion process, we need the following assumptions.
\begin{assumption} \label{ass:Time-Space Scalings}The discrete time step $\Delta t_N$ scales quadratically with the bandwidth parameter $\varepsilon_N$ 
$$\Delta t_N=c\varepsilon_N^2  \quad \text{for a fixed constant } c > 0. \text{ (Time-Space Scaling)}$$
\end{assumption}
Informally, we say that a Markov chain with transition matrix $P$ is a discrete observation of a continuous stochastic process $Y_T$ if its discrete transition operator matches the continuous process generator $L$ over an infinitesimal time step. Formally, we assume:
\begin{assumption}[Uniform discrete-generator consistency]
\label{ass:generator-consistency}
Define the empirical generator
\[
L_{N,\varepsilon_N}f(x_i)
\coloneqq
\frac{1}{\Delta t_N}
\left(
\sum_{j=1}^N P_{ij}f(x_j)-f(x_i)
\right).
\]
We assume that for every \(f\in C_c^\infty(\R^d)\),
\[
\sup_{x_i\in K}
\left|
L_{N,\varepsilon_N}f(x_i)-Lf(x_i)
\right|
\to 0
\]
as \(N\to\infty\) and \(\varepsilon_N\to 0\).
\end{assumption}

\begin{assumption}\label{ass:locality}
    We assume local support of the transition kernel on the
region of interest: for every compact $K\subset\mathbb{R}^d$, there
exists a compact neighborhood $K^+\supset K$ such that, for all
sufficiently large $N$,
\[
    x_i\in K,\quad P_{ij}>0
    \qquad\Longrightarrow\qquad
    x_j\in K^+.
\]
\end{assumption}
\begin{proposition}[Recovery of drift and diffusion from \(P\)]
\label{prop:p-to-sde}
Assume that $P=P_{N,\varepsilon_N}$ is row-stochastic and that
Assumptions~\cref{ass:Time-Space Scalings}
~\cref{ass:generator-consistency} and ~\cref{ass:locality} hold. Define
\[
\hat b_{0,N}(x_i)
=
\frac{1}{\Delta t_N}
\sum_{j=1}^N P_{ij}(x_j-x_i)\text{ and }
\hat D_N(x_i)
=
\frac{1}{\Delta t_N}
\sum_{j=1}^N
P_{ij}(x_j-x_i)(x_j-x_i)^\top.
\]
Then
\[
\max_{x_i\in K}
\|\hat b_{0,N}(x_i)-b_0(x_i)\|
\to 0 \text{ and }
\max_{x_i\in K}
\|\hat D_N(x_i)-D(x_i)\|
\to 0.
\]
\end{proposition}
\begin{proof}
Fix a compact set $K$ and let $K^+$ be the compact neighborhood
from \cref{ass:locality}. Choose a cutoff
$\chi\in C_c^\infty(\mathbb{R}^d)$ such that
$\chi\equiv 1$ on a neighborhood of $K^+$. For $a,b\in\{1,\ldots,d\}$ define
\[
    \widetilde q_a(x):=\chi(x)x_a,
    \qquad
    \widetilde q_{ab}(x):=\chi(x)x_ax_b .
\]
Then $\widetilde q_a,\widetilde q_{ab}\in C_c^\infty(\mathbb{R}^d)$,
so Assumption~\cref{ass:generator-consistency} applies. For $x_i\in K$ \cref{ass:locality} implies that every $x_j$ with $P_{ij}>0$
lies in $K^+$, where $\chi\equiv1$. Hence
\[
    L_{N,\varepsilon_N}\widetilde q_a(x_i)
    =
    L_{N,\varepsilon_N}q_a(x_i),
    \qquad
    L_{N,\varepsilon_N}\widetilde q_{ab}(x_i)
    =
    L_{N,\varepsilon_N}q_{ab}(x_i).
\]
Moreover, since $\chi\equiv1$ on a neighborhood of $K$,
\[
    L\widetilde q_a(x)=Lq_a(x),
    \qquad
    L\widetilde q_{ab}(x)=Lq_{ab}(x),
    \qquad x\in K.
\]
Therefore Assumption~\cref{ass:generator-consistency} may be applied
to the cutoff observables while yielding the desired generator
identities for $q_a$ and $q_{ab}$ on $K$. Since $P$ is
row-stochastic,
\[
L_{N,\varepsilon_N}q_a(x_i)
=
\frac{1}{\Delta t_N}
\sum_{j=1}^N
P_{ij}(x_{j,a}-x_{i,a})
=
\hat b_{0,N,a}(x_i).
\]
Moreover,
\[
Lq_a(x)=b_{0,a}(x).
\]
The drift convergence therefore follows directly from
Assumption~\cref{ass:generator-consistency}. Expanding the increments for the second moment gives 
\[
\begin{aligned}
L_{N,\varepsilon_N}q_{ab}(x_i)
&=
x_{i,a}\hat b_{0,N,b}(x_i)
+
x_{i,b}\hat b_{0,N,a}(x_i)
+
\hat D_{N,ab}(x_i).
\end{aligned}
\]
On the other hand, applying the diffusion generator yields
\[
Lq_{ab}(x)
=
x_a b_{0,b}(x)
+
x_b b_{0,a}(x)
+
D_{ab}(x).
\]
Consequently,
\[
\begin{aligned}
\hat D_{N, ab}(x_i)-D_{ab}(x_i)
&=
\bigl(
L_{N,\varepsilon_N}q_{ab}(x_i)
-
Lq_{ab}(x_i)
\bigr)
\\
&\quad
-
x_{i,a}
\bigl(
L_{N,\varepsilon_N}q_b(x_i)-Lq_b(x_i)
\bigr)
\\
&\quad
-
x_{i,b}
\bigl(
L_{N,\varepsilon_N}q_a(x_i)-Lq_a(x_i)
\bigr).
\end{aligned}
\]
Since $K$ is compact, its coordinates are uniformly bounded.
Uniform generator consistency therefore implies uniform convergence
of every entry of $\hat D_N$ to the corresponding entry of $D$.
Since the dimension is finite, this implies convergence in any
matrix norm.
\end{proof}
\paragraph{Generator consistency for kernels:}\label{app:Generator consistency}
The generator-consistency assumption \cref{ass:generator-consistency} is closely related to continuum limits of local graph random walks
\citep[Theorems~2--3]{tingAnalysisConvergenceGraph2011}, whose limiting
drift and diffusion are determined by the first two local moments
\citep[Lemma~3.9 and Theorem~3.12]{berryLocalKernelsGeometric2015a}.
For neighbourhood kernels, sufficient conditions
involve shrinking bandwidths, sufficiently dense sampling from a smooth
positive density, and controlled local bandwidth estimates. For the generator consistency assumption the row-normalized kernel,
including any velocity and lineage kernel, must satisfy
\[
m_h = ch^2 b_0 + o(h^2),
\qquad
C_h = ch^2 D + o(h^2),
\]
uniformly with higher-order displacement contributions divided
by $h^2$ going to zero as $h$ goes to zero. For the cosine and correlation constructions studied by \citet{liMathematicsRNAVelocity2021} gives a first-order limit on the time scale $h$, rather than a diffusion limit on the time
scale $h^2$. Their inner-product construction instead admits a
second-order limit. A possible
diffusive construction is to replace a fixed cosine weight
$\exp(\nu\cos\theta)$ by $\exp(\nu_h\cos\theta)$ with $\nu_h/h\to\nu_0$, which could be explored in future work.
FFM itself uses the transition moments directly, their interpretation
as drift and diffusion.
\paragraph{First-order velocity-kernel limits:}\label{first oder} Suppose that for some time scale \(\tau_h\to0\), the discrete generators \(L_h=(P_h-I)/\tau_h\) converge to
\[
Lf=b^\top\nabla f+\tfrac12\operatorname{Tr}(A\nabla^2f).
\]
Applying this convergence to coordinate and quadratic coordinate functions (as in \cref{prop:p-to-sde}) yields
\[
m_h/\tau_h\to b,\qquad C_h/\tau_h\to A.
\]
In particular, \citet{liMathematicsRNAVelocity2021} show that cosine RNA-velocity kernels possess a first-order limit with \(h=\sqrt{\varepsilon}\)
\[
(P_h-I)/h\to a(x)\widehat v(x)^\top\nabla ,
\]with \(a(x)>0\). Consequently \(m_h/h\to a(x)\widehat v(x)\), while \(C_h/h\to0\). Thus the first transition moment asymptotically recovers the RNA-velocity direction. For such kernels we therefore interpret \(C_h\) as a finite-scale measure of anisotropic transition accessibility rather than as an estimator of diffusion.
\paragraph{Extension of empirical drift and diffusion.}
The estimators in Proposition~\cref{prop:p-to-sde} are defined only at
the observed states. To evaluate the induced metric along continuous
curves, we extend them to a compact region
$K\subset\mathbb{R}^d$ containing all admissible paths. Let $K^+$ be
a compact neighborhood of $K$ such that
\[
K\subset \operatorname{int}(K^+).
\]
We assume that Assumption~\cref{ass:generator-consistency} holds on $K^+$, so that
Proposition~\cref{prop:p-to-sde} applies on this neighborhood.

\begin{definition}[Uniform nodewise consistency]
\label{def:nodewise-consistency}
Let $A\subset\mathbb{R}^d$ be compact and let $f$ be a continuous
vector- or matrix-valued field on $A$. Suppose that
$\widehat f_N(x_i)$ estimates $f(x_i)$ at each observed state
$x_i\in A$. We say that $\widehat f_N$ is uniformly nodewise
consistent for $f$ on $A$ if
\[
\max_{i:x_i\in A}
\left\|
    \widehat f_N(x_i)-f(x_i)
\right\|
\longrightarrow0
\]
along the asymptotic regime under consideration. For matrix-valued functions, any matrix norm may be used, since all
matrix norms are equivalent in finite dimension.
\end{definition}
In particular, Proposition~\cref{prop:p-to-sde} shows that
$\hat b_{0,N}$ and $\hat D_N$ are uniformly nodewise
consistent for $b_0$ and $D$, respectively, on $K^+$.

\begin{definition}[Kernel extension]
\label{def:p-kernel-fields}
Let $h\in C_c^\infty([0,\infty))$ be nonnegative, supported in
$[0,1]$, and strictly positive on a neighborhood of zero. Let
$\eta_N\to0$ and define
\[
k_{\eta_N}(x,y)
=
h\left(
    \frac{\|x-y\|^2}{\eta_N^2}
\right).
\]
Define the fill distance of the sample cloud relative to $K$ by
\[
r_N(K)
\coloneqq
\sup_{x\in K}
\min_{i:x_i\in K^+}
\|x-x_i\|.
\]
We assume that the samples become dense relative to the kernel
bandwidth:
\[
\frac{r_N(K)}{\eta_N}
\longrightarrow0.
\]
Since $h$ is positive near zero, this condition ensures that, for all
sufficiently large $N$, every $x\in K$ receives positive kernel
weight from at least one observed state in $K^+$. For such $N$, define
\[
w_{i,N}(x)
=
\frac{
    k_{\eta_N}(x,x_i)
}{
    \displaystyle
    \sum_{j:x_j\in K^+}
    k_{\eta_N}(x,x_j)
},
\qquad
x\in K,\quad x_i\in K^+,
\]
and extend the nodewise estimators by
\[
\widetilde b_{0,N}(x)
=
\sum_{i:x_i\in K^+}
w_{i,N}(x)\hat b_{0,N}(x_i)
\]
and
\[
\widetilde D_N(x)
=
\sum_{i:x_i\in K^+}
w_{i,N}(x)\hat D_N(x_i).
\]
Consequently, for sufficiently large $N$, the extended fields are smooth on a neighborhood of
$K$. Moreover, because every $\hat D_N(x_i)$ is positive
semidefinite, $\widetilde D_N(x)$ is also positive semidefinite.
\end{definition}

\begin{proposition}[Uniform convergence of the kernel-extended estimators]
\label{prop:p-smoothed-uniform}
Assume that $b_0$ and $D$ are continuous on $K^+$ and that
$\hat b_{0,N}$ and $\hat D_N$ are uniformly nodewise consistent on $K^+$ in the sense of
Definition~\cref{def:nodewise-consistency}. If
$\eta_N\longrightarrow0,
\frac{r_N(K)}{\eta_N}\longrightarrow0,$
then
\[
\sup_{x\in K}
\left\|
    \widetilde b_{0,N}(x)-b_0(x)
\right\|
\longrightarrow0 \text{ and }
\sup_{x\in K}
\left\|
    \widetilde D_N(x)-D(x)
\right\|
\longrightarrow0.
\]
\end{proposition}

\begin{proof} Modulus of continuity for a continuos function $f$ on $K^+$ is defined as:
\[
\omega_f(r)
\coloneqq
\sup_{\substack{
    x,y\in K^+\\
    \|x-y\|\le r
}}
\|f(x)-f(y)\|.
\]
Since $K^+$ is compact, $f$ is uniformly continuous and therefore
$\omega_f(r)\longrightarrow0 \text{ as }r\longrightarrow0.$ Consider first the drift field. Since the kernel weights are
nonnegative and sum to one,
\[
\begin{aligned}
\left\|
    \widetilde b_{0,N}(x)-b_0(x)
\right\|
&\leq
\sum_{i:x_i\in K^+}
w_{i,N}(x)
\left\|
    \hat b_{0,N}(x_i)-b_0(x_i)
\right\|
\\
&\quad+
\sum_{i:x_i\in K^+}
w_{i,N}(x)
\left\|
    b_0(x_i)-b_0(x)
\right\|.
\end{aligned}
\]
Whenever $w_{i,N}(x)>0$, the support condition on $h$ implies
\[
\|x_i-x\|\leq\eta_N.
\]
Hence
\[
\left\|
    \widetilde b_{0,N}(x)-b_0(x)
\right\|
\leq
\max_{i:x_i\in K^+}
\left\|
    \hat b_{0,N}(x_i)-b_0(x_i)
\right\|
+
\omega_{b_0}(\eta_N).
\]
Taking the supremum over $x\in K$, the first term converges to zero
by uniform nodewise consistency, while the second converges to zero
because $\eta_N\to0$.

Applying the same argument to the matrix-valued field $D$ gives
\[
\sup_{x\in K}
\left\|
    \widetilde D_N(x)-D(x)
\right\|
\leq
\max_{i:x_i\in K^+}
\left\|
    \hat D_N(x_i)-D(x_i)
\right\|
+
\omega_D(\eta_N)
\longrightarrow0.
\]
\end{proof}

\paragraph{Freidlin--Wentzell Theory.} Consider the small-noise diffusion
\[
\dd Y_T^\kappa
=
b_0(Y_T^\kappa)\,\dd t
+
\sqrt{2\kappa}\,D(Y_T^\kappa)^{1/2}\,\dd W_t ,
\]
where \(b_0\) is smooth. For an
absolutely continuous path \(\phi:[0,T]\to\mathbb R^d\), the
Freidlin--Wentzell (FW) \citep{freidlinRandomPerturbationsDynamical2012a} action is defined as:
\begin{align}\label{eq: Fixed-Time FW}
S_T(\phi)
=
\frac14
\int_0^T
(\dot\phi_t-b_0(\phi_t))^\top
D(\phi_t)^{-1}
(\dot\phi_t-b_0(\phi_t))
\,\dd t .
\end{align}
FW theory gives a path-space large-deviation principle with rate
functional $S_T$. Informally speaking, among paths satisfying the same endpoint constraints,
minimizers of $S_T$ are the asymptotically most probable transition
paths (See Theorem 4.1~\citet{freidlinRandomPerturbationsDynamical2012a}).
\\
\\
The fixed-time action $S_T$ depends on the prescribed transition duration $T$. Minimizing over $T$ yields the quasipotential
\[
V(x_0,x_1)
\coloneqq
\inf_{T>0}
\inf_{\substack{
    \phi: \phi(0)=x_0,\ \phi(T)=x_1
}}
S_T(\phi).
\]
By Proposition~2.1 and Equation~(2.14) of
\citet{heymannGeometricMinimumAction2008a},
\begin{equation}
\label{eq:quasipotential}
V(x_0,x_1)
=
\inf_{T>0}
\inf_{\substack{
    \phi:
\phi(0)=x_0,\ \phi(T)=x_1
}}
S_T(\phi)
=
\frac12
\inf_{\substack{
    \gamma:
    \gamma(0)=x_0,\ \gamma(1)=x_1
}}
S_{\mathrm{FW}}(\gamma).
\end{equation}
where $S_{\mathrm{FW}}(\gamma) = \int_0^1F^{\mathrm{FW}}(\gamma(t), \dot \gamma(t))\dd t$ with
\[F^{\mathrm{FW}}(x,v) \coloneqq \Bigl(v^\top \underbrace{\Vert{}b_0(x)\Vert{}_{D(x)^{-1}}^2 D(x)^{-1}}_{G_{\mathrm{FW}}(x)} v\Bigl)^{1/2} + v^\top \underbrace{\left( -D(x)^{-1}b_0(x) \right)}_{\beta_{\mathrm{FW}}(x)}.\]
This has almost the form of a Finsler length functional \cref{eq:length functional}, however, for $b_0(x)\neq0$,
$\|\beta_{\mathrm{FW}}(x)\|_
{G_{\mathrm{FW}}(x)^{-1}}
=1.$
Thus $F^{\mathrm{FW}}$ is a degenerate Randers-type structure rather than Finsler metric (See \cref{lem: Finsler admissibility}).
\\
\\
$S_{\mathrm{FW}}$ defines the
parameterization-invariant FW action whose
minimizers are the minimum-action transition curves. Let $\gamma^*$ be the unique minimizer (up  reparametrization) among curves connecting $x_0$ to $x_1$. We define the optimal transition time scale
\[
T^* = \int_0^1 \frac{\dd t}{\Lambda(\gamma^*(t), \dot{\gamma}^*(t))}, \qquad \text{where } \Lambda(x, v) = \sqrt{\frac{b_0(x)^\top D(x)^{-1} b_0(x)}{v^\top D(x)^{-1} v}} ,
\]
and denote by $\tau_\delta(X^\kappa) = \inf \{ t > 0 \mid Y_T^\kappa \in \bar{B}_\delta(x_1) \}$ the first hitting time of a closed $\delta$-ball centered at $x_1$.
Under the hypotheses of
\citet[Proposition~2.3]{heymannGeometricMinimumAction2008a},
including uniqueness of the minimizer and of the corresponding constrained fixed-time minimizers, together with the
stated compactness and convergence conditions, the following
conditional concentration statement holds, for any error tolerance $\eta > 0$,
\[
\lim_{\delta \to 0^+ \\ T \to T^*} \liminf_{\kappa \to 0} \, \mathbb{P}_{x_0} \!\left( \rho_F \!\left( X^\kappa\big|_{[0, \tau_\delta(X^\kappa)]}, \, \gamma^* \right) < \eta \,\,\middle|\,\, \tau_\delta(X^\kappa) \le T \right) = 1 ,
\]
where $\rho_F$ is the Fréchet distance between curves. Informally speaking, conditioned on the system completing the transition from $x_0$ into a small neighborhood of $x_1$ within a time frame close to $T^*$, the realization $Y_T^\kappa$ concentrates with probability $1$ in an arbitrarily narrow geometric tube around the minimizer $\gamma^*$ as $\kappa \to 0$.
\paragraph{Empirical Finsler minimizers converge to geometric FW
transition curves:} We next show that minimizers of the regularized empirical Finsler
action converge to the population geometric FW minimizer, which has
the small-noise transition interpretation established above.
\begin{assumption}[Compact path class]
\label{ass:compact-path-class}
Let \(\Gamma=\{\gamma_\theta:\theta\in\Theta, \Vert{}\dot{\gamma}(t)\Vert{} \le M)\}\) be a class of smooth curves with
\(\gamma_\theta(0)=x_0\), \(\gamma_\theta(1)=x_1\), \(\Theta\) compact, and
\((t,\theta)\mapsto(\gamma_\theta(t),\dot\gamma_\theta(t))\) continuous. Set
\(K_\Gamma=\overline{\bigcup_{\theta\in\Theta}\gamma_\theta([0,1])}\). We assume that $K_\Gamma\subset K$, where $K$ is the compact set on
which the uniform convergence in
Proposition~\cref{prop:p-smoothed-uniform} holds.
\end{assumption}
For the unregularized geometric Freidlin--Wentzell interpretation
in this subsection, we additionally assume that $D$ is uniformly
positive definite on the compact region $K^+$ containing the admissible paths.
For numerical stability and to ensure uniform positive definiteness at finite \(N\), we use the ridge-regularized estimator
\[
\tilde D_{\rho_N}(x)
\coloneqq
\tilde D_N(x)+\rho_N I,
\]
where \(\rho_N>0\) and \(\rho_N\to0\) as \(N\to\infty\).
Furthermore, we define the regularized FW metric as: 
\[
\tilde F^{\mathrm{FW}}_{N, \lambda}(x,v)
\coloneqq
\sqrt{
    v^\top\tilde D_{\rho_N}(x)^{-1}v
}
\,
\sqrt{\|\tilde b_{0,N}(x)\|^2_
{\tilde D_{\rho_N}(x)^{-1}}+\lambda_N^2}
-
v^\top
\tilde D_{\rho_N}(x)^{-1}
\tilde b_{0,N}(x).
\]
, note that this ensures that $\tilde F^{\mathrm{FW}}_{N, \lambda}$ is a $P$-derived Randers \cref{def:Pderived} metric thus a Finsler structure and hence a for some sequence $\lambda_N >0$ which converges to $0$ as $N \rightarrow \infty$.

\begin{theorem}[Consistency of regularized empirical geometric FW
minimizers]
\label{thm:empirical-gmam}
Assume Assumption~\cref{ass:compact-path-class} and the assumptions
of Proposition~\cref{prop:p-smoothed-uniform}. Let
$\rho_N>0$ and $\lambda_N>0$ satisfy $\rho_N\to0,
\lambda_N\to0$. Define
\[
\tilde S_N(\gamma)
\coloneqq
\int_0^1
\tilde{F}^{\mathrm{FW}}_{N,\lambda_N}
\left(
    \gamma(t),
    \dot{\gamma}(t)
\right)
\,\dd t,
\]
and let
\[
\tilde\gamma_N^*
\in
\arg\min_{\gamma\in\Gamma}
\tilde S_N(\gamma).
\]
Then:
\begin{enumerate}
    \item
    \[
    \sup_{\gamma\in\Gamma}
    \left|
        \tilde S_N(\gamma)
        -
         S_{\mathrm{FW}}(\gamma)
    \right|
    \longrightarrow0.
    \]

    \item
    If $\mathcal S_{\mathrm{FW}}$ has a unique minimizer
    $\gamma^*$ over $\Gamma$, up to reparameterization, then
    \[
    \rho_F
    \left(
        \tilde\gamma_N^*,
        \gamma^*
    \right)
    \longrightarrow0.
    \]
\end{enumerate}
Consequently, the regularized empirical Finsler minimizer converges
to the population FW minimizer within $\Gamma$. If
$\gamma^*$ is also the global geometric FW minimizer connecting
$x_0$ to $x_1$, and the hypotheses of
\citet[Proposition~2.3]{heymannGeometricMinimumAction2008a} hold,
then $\gamma^*$ is the asymptotically most probable transition
curve in the small-noise sense stated above, and
$\tilde\gamma_N^*$ converges to this curve.
\end{theorem}

\begin{proof}
By Proposition~\cref{prop:p-smoothed-uniform}, we have uniform convergence on \(K_\Gamma\):
\[
\sup_{x\in K_\Gamma}\|\tilde b_{0,N}(x)-b_0(x)\|\to 0,
\qquad
\sup_{x\in K_\Gamma}\|\tilde D_N-D(x)\|\to 0 .
\]
Since \(\rho_N\to0\), we also have
\[
\sup_{x\in K_\Gamma}
\|\tilde D_{\rho_N}(x)-D(x)\|
\le
\sup_{x\in K_\Gamma}
\|\tilde D_N(x)-D(x)\|
+
\rho_N
\to0.
\]
Since \(D\) is uniformly positive definite on \(K_\Gamma\), and since matrix inversion is continuous on the set of uniformly positive definite matrices, it follows that
\[
\sup_{x\in K_\Gamma}\|\tilde D_{\rho_N}(x)^{-1}-D(x)^{-1}\|\to 0 .
\]

By Assumption~\cref{ass:compact-path-class}, all curves in \(\Gamma\) have uniformly bounded velocity:
\[
\|\dot\gamma(t)\|\le M
\qquad
\text{for all }\gamma\in\Gamma,\ t\in[0,1].
\]
Therefore all pairs \((\gamma(t),\dot\gamma(t))\) lie in the compact set $K_\Gamma\times \overline{B_M(0)}$ .

On this compact set, the map
\[
(b,D,v, \lambda_N)
\mapsto
\sqrt{v^\top D^{-1}v}\,\sqrt{\|b\|^2_{D^{-1}}+ \lambda_N^2}
-
v^\top D^{-1}b
\]
is continuous. The uniform convergence of \(\tilde b_{0,N}\), \(\tilde D_{\rho_N}\) and $\lambda_N\rightarrow 0$ implies
\[
\sup_{x\in K_\Gamma,\ \|v\|\le M}
\left|
\tilde F^{\rm FW}_{N, \lambda_N}(x,v)-F^{\rm FW}(x,v)
\right|
\to 0 .
\]
Consequently, for every \(\gamma\in\Gamma\),
\[
|\tilde S_N(\gamma)-S_{\rm FW}(\gamma)|
\le
\int_0^1
\left|
\tilde F^{\rm FW}_{N, \lambda_N}(\gamma(t),\dot\gamma(t))
-
F^{\rm FW}(\gamma(t),\dot\gamma(t))
\right|
\,dt .
\]
Taking the supremum over \(\gamma\in\Gamma\), we obtain
\[
\sup_{\gamma\in\Gamma}
|\tilde S_N(\gamma)-S_{\rm FW}(\gamma)|
\le
\sup_{x\in K_\Gamma,\ \|v\|\le M}
\left|
\tilde F^{\rm FW}_{N, \lambda_N}(x,v)-F^{\rm FW}(x,v)
\right|
\to 0 .
\]
This proves the first claim. Let \(\gamma^*\) be the unique minimizer of \(S_{\rm FW}\) over \(\Gamma\), up to reparameterization. Fix \(r>0\), and define
\[
\Gamma_r
=
\{\gamma\in\Gamma:\rho_F(\gamma,\gamma^*)\ge r\}.
\]
Because \(\Gamma\) is compact and \(S\) is continuous on \(\Gamma\), the minimum of \(S\) over \(\Gamma_r\) is attained. Since \(\gamma^*\) is the unique minimizer, there is a positive gap
\[
\Delta_r
\coloneqq
\inf_{\gamma\in\Gamma_r} S_{\rm FW}(\gamma)-S_{\rm FW}(\gamma^*)
>
0 .
\]
By the uniform convergence just proved, for sufficiently large \(N\),
\[
\sup_{\gamma\in\Gamma}
|\tilde S_N(\gamma)-S_{\rm FW}(\gamma)|
<
\frac{\Delta_r}{3}.
\]
For any \(\gamma\in\Gamma_r\), this gives
\[
\tilde S_N(\gamma)
\ge
S_{\rm FW}(\gamma)-\frac{\Delta_r}{3}
\ge
S_{\rm FW}(\gamma^*)+\Delta_r-\frac{\Delta_r}{3}
=
S_{\rm FW}(\gamma^*)+\frac{2\Delta_r}{3}.
\]
On the other hand,
\[
\tilde S_N(\gamma^*)
\le
S_{\rm FW}(\gamma^*)+\frac{\Delta_r}{3}.
\]
Therefore
\[
\tilde S_N(\gamma)>\tilde S_N(\gamma^*)
\qquad
\text{for all }\gamma\in\Gamma_r .
\]
Hence no minimizer \(\tilde\gamma^*\in\arg\min_{\gamma\in\Gamma}\tilde S_N(\gamma)\) can lie in \(\Gamma_r\). Thus,
\[
\rho_F(\tilde\gamma^*,\gamma^*)<r
\]
for all sufficiently large \(N\). Since \(r>0\) was arbitrary,
\[
\rho_F(\tilde\gamma^*,\gamma^*)\to 0 .
\]

Finally, suppose that $\gamma^*$ is also the global geometric FW
minimizer connecting $x_0$ to $x_1$, and that the hypotheses of
\citet[Proposition~2.3]{heymannGeometricMinimumAction2008a} hold.
The conditional concentration result above then identifies $\gamma^*$ as the asymptotically most probable transition curve in
the small-noise limit. Since
\[
\rho_F
\left(
    \tilde\gamma_N^*,
    \gamma^*
\right)
\longrightarrow0,
\]
the regularized empirical Finsler minimizer converges to this
transition curve.
\end{proof}
\begin{remark}[Fixed-ridge geometry]
\label{rem:fixed-ridge}
The positive-definiteness assumption above is needed only for passing to the unregularized limit $\rho_N\to0$ and recovering the
geometric FW action of the original diffusion.
For any fixed $\bar\rho>0$, the regularized population tensor
$D_{\bar\rho}(x):=D(x)+\bar\rho I$ satisfies
is unifromly positive definite. Hence,
Proposition~\cref{prop:p-smoothed-uniform} implies
\[
\sup_{x\in K}
\left\|
(\widetilde D_N(x)+\bar\rho I)^{-1}
-
(D(x)+\bar\rho I)^{-1}
\right\|
\longrightarrow 0.
\]
Consequently, for fixed $\bar\rho>0$ and $\bar\lambda>0$,
the same argument as in Theorem~\cref{thm:empirical-gmam}
gives uniform convergence of the empirical regularized action
to the corresponding population regularized action, without
requiring $D$ itself to be uniformly positive definite.
\end{remark}
\begin{remark}
For a common graph time step $\Delta t > 0$, the
diffusion-normalized moments are
$\widehat b = m/\Delta t$ and $\widehat D = C/\Delta t$.
Writing $F_{\rho,\lambda}[m,C]$ for the regularized
metric constructed from the moments $m,C$, direct
substitution gives
\[
F_{\rho,\lambda}[m,C](x,v)
=
F_{\rho/\Delta t,\,\lambda/\sqrt{\Delta t}}
\left[
\frac{m}{\Delta t},
\frac{C}{\Delta t}
\right](x,v).
\]
Thus, raw and diffusion-normalized moments yield
exactly the same metric after rescaling the
regularization parameters, and we absorb the
time-step normalization into these hyperparameters
in the implementation.
The vanishing-regularization assumptions in the
consistency theorem apply to the normalized
parameters; in raw-moment notation, they require
$\rho_N/\Delta t_N \to 0$ and
$\lambda_N/\sqrt{\Delta t_N} \to 0$.
\end{remark}

\paragraph{Additional Calculation:}\begin{proposition}[Gaussian probability path]\label{prop: vectorfield calc}
Let $\mu_t\in\mathbb{R}^d$ be continuously differentiable and let
$G_t\in\mathrm{SPD}(d)$ be a continuously differentiable.
Define, for $t\in(0,1)$, $p_t(x \mid x_0, x_1)=\mathcal N\!\left(x;\mu_t,\,t(1-t)G_t\right).$
Then the vector field
\begin{align*}
\tilde{u}_t(x\mid x_0, x_1)
=
\dot{\mu}_t
+
\frac{1-2t}{2t(1-t)}(x-\mu_t)
+
\frac12\dot G_tG_t^{-1}(x-\mu_t)
\end{align*}
satisfies $\frac{dx}{dt}= \tilde{u}_t(x_t \mid x_0, x_1) $ and $x_t \sim p_t(\cdot \mid x_0, x_1)$ for all $t \in (0,1).$
\end{proposition}

\begin{proof}
Let $B_t$ be a differentiable matrix factor satisfying
\begin{align*}
B_tB_t^\top=G_t,
\qquad
\dot B_tB_t^{-1}
=
\frac12\dot G_tG_t^{-1},
\end{align*}
and let $\varepsilon\sim\mathcal N(0,I)$. Such a metric exists since $G_t$ is symmetric positive definite (SPD) and by using existence and uniqueness of ODEs. A sample from $p_t$ can then be written as
\begin{align*}
x_t
=
\mu_t
+
\sqrt{t(1-t)}\,B_t\varepsilon.
\end{align*}
Solving for $\varepsilon$ gives
\begin{align*}
\varepsilon
=
\frac{1}{\sqrt{t(1-t)}}
B_t^{-1}(x_t-\mu_t).
\end{align*}
Differentiating $x_t$ with respect to $t$ yields
\begin{align*}
\dot x_t
&=
\dot\mu_t
+
\frac{1-2t}{2\sqrt{t(1-t)}}B_t\varepsilon
+
\sqrt{t(1-t)}\,\dot B_t\varepsilon \\
&=
\dot\mu_t
+
\frac{1-2t}{2t(1-t)}(x_t-\mu_t)
+
\dot B_tB_t^{-1}(x_t-\mu_t) \\
&=
\dot\mu_t
+
\frac{1-2t}{2t(1-t)}(x_t-\mu_t)
+
\frac12\dot G_tG_t^{-1}(x_t-\mu_t).
\end{align*}
Hence $\dot x_t=u_t(x_t)$. Since
$x_t=\mu_t+\sqrt{t(1-t)}B_t\varepsilon$ has distribution
$\mathcal N(\mu_t,t(1-t)G_t)$ by construction, the flow generated by $u_t$
has marginals $p_t$ for all $t\in(0,1)$.
\end{proof}
\begin{remark}
    Note that in the case where $G_t = \Id$ and $\mu_t = tx_1 +(1-t)x_0 $we recover the formula from \citet{tongSimulationFreeSchrodingerBridges2024a}.
\end{remark}
\section{Experiments}\label{appendix C}
\paragraph{Experimental protocol:} All models get the same networks that uses four hidden layers of width 64, SiLU activations, and a 32-dimensional sinusoidal time embedding. We use Adam with learning rate \(2\times10^{-3}\), batch size 256, and gradient clipping at 5. Where applicable, geodesic training uses 2,500 iterations, followed by 1,500 flow-training iterations, stochastic training uses an additional 1,500 iterations. Default coupling settings use eight quadrature samples and entropic regularization equal to 0.1 times the median transport cost, with benchmark-specific overrides stated separately. Results report mean $\pm$ population standard deviation across three seeds, except Erythroid, which uses five. Hyperparameters were selected using a validation split of $15$ percent and using the mean of mid and end $W_2$.
\paragraph{Kernel extension of transition moments:} \label{app:Kernel smoothing}
The first and second transition moments $m(x_i)$ and $C(x_i)$
are defined only at the observed states $\{x_i\}_{i=1}^N$.
To evaluate the Finsler metric along continuous paths, we
extend both moments using normalized non-negative kernel
smoothing. Let
$k_h:\mathbb{R}^d\times\mathbb{R}^d\to\mathbb{R}_{\geq 0}$
be a smooth kernel with bandwidth $h>0$. Assuming that
$\sum_{\ell=1}^N k_h(x,x_\ell)>0$ throughout the region of
interest, define
\[
w_i(x)
\coloneqq
\frac{k_h(x,x_i)}
{\sum_{\ell=1}^N k_h(x,x_\ell)}.
\]
The extended moment fields are
\begin{equation}
\label{eq:moment_extensions}
m(x)
\coloneqq
\sum_{i=1}^N w_i(x)m(x_i),
\qquad
C(x)
\coloneqq
\sum_{i=1}^N w_i(x)C(x_i).
\end{equation}

We choose kernel smoothing for its simplicity: it requires
no additional model training and preserves the positive
semidefiniteness of the second-moment field. Indeed, the
weights are non-negative and sum to one, so $C(x)$ is a
convex combination of positive semidefinite matrices.
Both extended fields are smooth wherever the denominator
is positive. Adding the ridge term
$C_\rho(x)=C(x)+\rho I$, with $\rho>0$, then ensures
positive definiteness for the subsequent metric construction.
\paragraph{Sheet Dataset:} \label{app: Sheet}We sample $N=4000$ points from coordinates $u\sim\mathcal U(-1.5,1.5)$ and $w\sim\mathcal U(-0.45,0.45)$ and embed them on the saddle
\[
z(u,w)=\frac{1}{2}\left[\left(\frac{3u}{2}\right)^2-\frac12\left(\frac{100w}{45}\right)^2\right],
\]
with Gaussian noise of standard deviation $0.012$. The prescribed transition probability matrix $P$ forms an hourglass: trajectories are driven towards $u\approx0$ and go outward after crossing. Source and target states are defined by $w\leq-0.13$ and $w\geq0.13$, respectively.
\\
\\
$P$ is constructed bye directed $15$-NN kernel combining spatial proximity with alignment to a prescribed vector field ($\lambda=6$), followed by row normalization. Source and target points are split $70/15/15$ into train/validation/test sets, whereas the withheld region is divided equally between validation and test. Hyperparameters are selected using the average of endpoint $W_2$ and intermediate-region $W_2$ on the validation set. We also perform a metric ablation where we compare only keeping the first and second moments \cref{fig:sheet_appendix_ablation}.
\begin{figure*}[t]
    \centering
    \begin{subfigure}[t]{0.47\textwidth}
        \vspace{0pt}
        \centering
       \makebox[\linewidth][c]{%
        \includegraphics[width=0.40\linewidth]{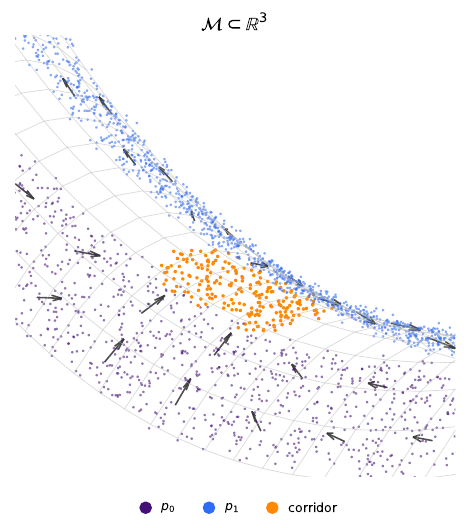}}
        \caption{
            Sheet benchmark.
        }
        \label{fig:sheet_appendix_dataset}
    \end{subfigure}
    \hfill
    \begin{subfigure}[t]{0.50\textwidth}
        \vspace{0pt}
        \centering
        \small
        \setlength{\tabcolsep}{3.5pt}
        \resizebox{\linewidth}{!}{%
        \begin{tabular}{lccrrr}
\toprule
$G$ & Drift & $c_{\mathrm{eff}}$ & mid $W_2\downarrow$ & end $W_2\downarrow$ & $W_{2,F}\downarrow$ \\
\midrule
$I$ & no  & 0.001 & 0.7165 $\pm$ 0.0136 & \textbf{0.1751 $\pm$ 0.0271} & 1.0840 $\pm$ 0.0371 \\
$I$ & yes & 0.958 & 0.2541 $\pm$ 0.0375 & 0.2558 $\pm$ 0.0509 & 0.3676 $\pm$ 0.0773 \\
$C_\rho^{-1}$ & no  & 0.001 & 0.6672 $\pm$ 0.0160 & 0.2025 $\pm$ 0.0079 & 0.9437 $\pm$ 0.1068 \\
$C_\rho^{-1}$ & yes & 0.958 & \textbf{0.1588 $\pm$ 0.0096} & 0.2343 $\pm$ 0.0292 & \textbf{0.2892 $\pm$ 0.0090} \\
\midrule
\emph{sampling floor} & & & \emph{0.0646} & \emph{0.1896} & \emph{0.2812} \\
\bottomrule
\end{tabular}%
        }
        \caption{
            Metric ablation.
        }
        \label{fig:sheet_appendix_ablation}
    \end{subfigure}
    \vspace{0.7em}
    \begin{subfigure}[t]{0.98\textwidth}
        \centering
        \includegraphics[width=\linewidth]{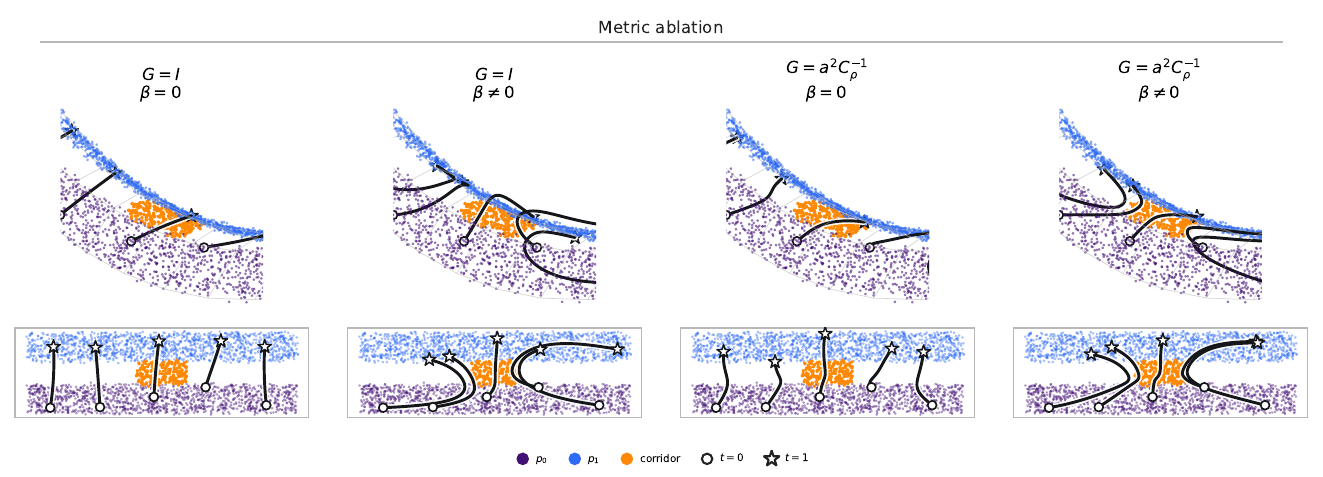}
        \caption{
             Trajectories generated by the different methods.
        }
        \label{fig:sheet_appendix_pat}
    \end{subfigure}

    \caption{
        \textbf{Sheet benchmark and additional results.}
        \textbf{(a)} Overview of the synthetic Sheet benchmark and its
        directed transition structure.
        \textbf{(b)} Ablation of the directional and anisotropic components
        of the Finsler metric. Directionality substantially improves
        intermediate recovery, with the full metric obtaining the lowest
        midpoint errors.
        \textbf{(c)} Qualitative comparison of generated trajectories.
    }
    \label{fig:sheet_appendix}
\end{figure*}
\paragraph{Pancreas Dataset:}
Hyperparameters are selected independently for
each representation using a validation subset disjoint from the reported withheld cells. Runtime and  sensitivity to $\rho$ and $\lambda$
are reported in \cref{tab:pancreas_runtimi} and \cref{tab:pancreas_sensitivity}. We additionally evaluate sensitivity to temporal alignment, the amount of intermediate-state support used to construct the geometry, and choices in the stochastic construction. To reduce dependence on the benchmark convention $t=1/2$, we evaluate $W_2$ over a fixed 11-point grid spanning $t\in[1/3,2/3]$ and report $\overline W_2^{[1/3,2/3]}$, the mean over this grid. We further report $W_{2,F}$ across representations, vary the fraction of intermediate-band cells available for metric construction, and ablate the conditional-path covariance and inference noise used by stochastic FFM. Unless otherwise stated, all Pancreas results in \cref{fig:paths_and_quantitative_results} use no intermediate-band cells for metric construction. \cref{fig:pancreas_robustness} is an additional sensitivity analysis that deliberately varies this fraction.
\begin{figure}[p]
\centering
\scriptsize
\renewcommand{\arraystretch}{0.92}

\textbf{(a) Sensitivity to model time}

\vspace{0.5mm}

\setlength{\tabcolsep}{2.8pt}
\resizebox{\textwidth}{!}{%
\begin{tabular}{lrrrrr}
\toprule
Method & $W_2(1/3)$ & $W_2(0.433)$ & $W_2(0.567)$ & $W_2(2/3)$ & $\overline W_2^{[1/3,2/3]}$ \\
\midrule
\multicolumn{6}{l}{\textit{$d=2$}} \\
FFM, deterministic & 1.9878 $\pm$ 0.0631 & 1.8184 $\pm$ 0.0326 & 1.4429 $\pm$ 0.0195 & 1.1770 $\pm$ 0.0919 & 1.6175 $\pm$ 0.0348 \\
FFM, with noise & 2.0052 $\pm$ 0.0705 & 1.8399 $\pm$ 0.0401 & 1.4723 $\pm$ 0.0601 & 1.2148 $\pm$ 0.0851 & 1.6434 $\pm$ 0.0555 \\
OT-CFM & 1.8508 $\pm$ 0.0172 & 1.7952 $\pm$ 0.0182 & 1.7607 $\pm$ 0.0195 & 1.7682 $\pm$ 0.0201 & 1.7869 $\pm$ 0.0230 \\
MFM / LAND & 1.5201 $\pm$ 0.0119 & 1.4878 $\pm$ 0.0117 & 1.4438 $\pm$ 0.0108 & 1.4002 $\pm$ 0.0098 & 1.4642 $\pm$ 0.0131 \\
Curly-FM & 1.0208 $\pm$ 0.1277 & 0.9952 $\pm$ 0.1378 & 0.9480 $\pm$ 0.1457 & 0.9009 $\pm$ 0.1505 & 0.9686 $\pm$ 0.1725 \\
\midrule
\multicolumn{6}{l}{\textit{$d=20$}} \\
FFM, deterministic & 11.5969 $\pm$ 0.1999 & 10.7814 $\pm$ 0.1811 & 10.3610 $\pm$ 0.1440 & 10.8331 $\pm$ 0.1471 & 10.7413 $\pm$ 0.1927 \\
FFM, with noise & 11.5392 $\pm$ 0.1640 & 10.7188 $\pm$ 0.1270 & 10.3671 $\pm$ 0.0477 & 10.9998 $\pm$ 0.2163 & 10.7355 $\pm$ 0.0560 \\
OT-CFM & 13.0727 $\pm$ 0.1325 & 12.5476 $\pm$ 0.0959 & 12.2577 $\pm$ 0.0231 & 12.3903 $\pm$ 0.0406 & 12.4957 $\pm$ 0.0681 \\
MFM / LAND & 12.8496 $\pm$ 0.1074 & 12.3470 $\pm$ 0.0706 & 11.8851 $\pm$ 0.0180 & 11.7331 $\pm$ 0.0277 & 12.1656 $\pm$ 0.0522 \\
Curly-FM & 12.8636 $\pm$ 0.0503 & 12.3801 $\pm$ 0.0687 & 12.2108 $\pm$ 0.0760 & 12.4096 $\pm$ 0.1064 & 12.3920 $\pm$ 0.0687 \\
\midrule
\multicolumn{6}{l}{\textit{$d=50$}} \\
FFM, deterministic & 13.2623 $\pm$ 0.1442 & 12.7423 $\pm$ 0.0966 & 12.5789 $\pm$ 0.0240 & 12.9036 $\pm$ 0.0495 & 12.7715 $\pm$ 0.0719 \\
FFM, with noise & 13.2718 $\pm$ 0.1524 & 12.7538 $\pm$ 0.1104 & 12.5768 $\pm$ 0.0639 & 12.8742 $\pm$ 0.0975 & 12.7723 $\pm$ 0.0952 \\
OT-CFM & 14.1359 $\pm$ 0.1441 & 13.6933 $\pm$ 0.1242 & 13.4941 $\pm$ 0.0923 & 13.6688 $\pm$ 0.0816 & 13.6811 $\pm$ 0.1314 \\
MFM / LAND & 13.8332 $\pm$ 0.1526 & 13.2557 $\pm$ 0.0966 & 12.7979 $\pm$ 0.0593 & 12.7353 $\pm$ 0.0364 & 13.0996 $\pm$ 0.1029 \\
Curly-FM & 13.8395 $\pm$ 0.1113 & 13.4287 $\pm$ 0.1521 & 13.2923 $\pm$ 0.2020 & 13.4719 $\pm$ 0.2307 & 13.4440 $\pm$ 0.2069 \\
\bottomrule
\end{tabular}%
}

\vspace{2mm}

\begin{minipage}[t]{0.43\textwidth}
\centering
\textbf{(b) Finsler-Wasserstein recovery}

\vspace{0.5mm}

\setlength{\tabcolsep}{2.3pt}
\resizebox{\linewidth}{!}{%
\begin{tabular}{lrrrr}
\toprule
Method & $d=2$ & $d=10$ & $d=20$ & $d=50$ \\
\midrule
OT-CFM & 0.6420 $\pm$ 0.0059 & 10.9311 $\pm$ 0.0770 & 10.9277 $\pm$ 0.0412 & 15.5900 $\pm$ 0.2287 \\
MFM / LAND & \textbf{0.5399 $\pm$ 0.0024} & \textbf{9.4377 $\pm$ 0.0824} & 10.4768 $\pm$ 0.0637 & 14.6185 $\pm$ 0.1085 \\
Curly-FM & 0.7493 $\pm$ 0.0258 & 10.7482 $\pm$ 0.1305 & 11.0235 $\pm$ 0.0520 & 15.0427 $\pm$ 0.1562 \\
FFM, deterministic & 0.5681 $\pm$ 0.0100 & 9.5125 $\pm$ 0.5398 & \textbf{8.9033 $\pm$ 0.0400} & 13.9368 $\pm$ 0.1217 \\
FFM, with noise & 0.5730 $\pm$ 0.0082 & 9.6113 $\pm$ 0.5379 & 8.9213 $\pm$ 0.0453 & \textbf{13.7800 $\pm$ 0.0861} \\
\midrule
\emph{floor} & \emph{0.4162} & \emph{5.3461} & \emph{4.3608} & \emph{7.2123} \\
\bottomrule
\end{tabular}%
}

\vspace{0.5mm}

{\tiny $W_{2,F}(t=1/2)\downarrow$.}

\end{minipage}
\hfill
\begin{minipage}[t]{0.55\textwidth}
\centering
\textbf{(c) Intermediate support used for metric construction}

\vspace{0.5mm}

\setlength{\tabcolsep}{2.6pt}
\resizebox{\linewidth}{!}{%
\begin{tabular}{lccc}
\toprule
Method & $0\%$ & $5\%$ & $35\%$ \\
\midrule
\multicolumn{4}{l}{\textit{$d=2$}} \\
OT-CFM & 1.7607 $\pm$ 0.0262 & 1.7622 $\pm$ 0.0392 & 1.7666 $\pm$ 0.0297 \\
FFM, deterministic & 1.3833 $\pm$ 0.1410 & 1.3313 $\pm$ 0.0627 & 1.0927 $\pm$ 0.2477 \\
FFM, with noise & 1.3723 $\pm$ 0.1505 & 1.3525 $\pm$ 0.0420 & 1.1013 $\pm$ 0.2611 \\
MFM / LAND & 1.4506 $\pm$ 0.0025 & 1.2412 $\pm$ 0.5283 & 0.6878 $\pm$ 0.0378 \\
Curly-FM & 0.9751 $\pm$ 0.1704 & 0.9360 $\pm$ 0.1045 & 1.5865 $\pm$ 0.1369 \\
\midrule
\multicolumn{4}{l}{\textit{$d=20$}} \\
OT-CFM & 12.3291 $\pm$ 0.0685 & 12.3264 $\pm$ 0.0445 & 12.3491 $\pm$ 0.0850 \\
FFM, deterministic & 10.4208 $\pm$ 0.1028 & 8.2742 $\pm$ 0.3540 & 8.1989 $\pm$ 0.0481 \\
FFM, with noise & 10.3679 $\pm$ 0.0641 & 8.4860 $\pm$ 0.3098 & 8.2513 $\pm$ 0.2160 \\
MFM / LAND & 12.0368 $\pm$ 0.0696 & 11.8784 $\pm$ 0.0338 & 11.7247 $\pm$ 0.1739 \\
Curly-FM & 12.2067 $\pm$ 0.0844 & 12.2163 $\pm$ 0.1098 & 12.2431 $\pm$ 0.0665 \\
\midrule
\multicolumn{4}{l}{\textit{$d=50$}} \\
OT-CFM & $13.5319\pm0.0480$ & $13.5269\pm0.0642$ & $13.5577\pm0.1075$ \\
FFM, deterministic & $\mathbf{12.5043\pm0.0544}$ & $\mathbf{12.1225\pm0.1252}$ & $\mathbf{11.3633\pm0.1356}$ \\
MFM / LAND & $13.5584\pm0.0396$ & $13.5208\pm0.0275$ & $13.5689\pm0.1239$ \\
Curly-FM & $13.3220\pm0.1927$ & $13.3291\pm0.2501$ & $13.2897\pm0.1116$ \\
\bottomrule
\end{tabular}%
}

\vspace{0.5mm}

{\tiny $W_2(t=1/2)\downarrow$.}

\end{minipage}

\vspace{1mm}

\caption{
\textbf{Pancreas robustness analyses.}
(a) Intermediate-population recovery over model time. We show representative
points from the fixed 11-point grid over $[1/3,2/3]$ together with
$\overline W_2^{[1/3,2/3]}$, the mean over all 11 points.
(b) Geometry-dependent Wasserstein distance across representations.
(c) Sensitivity to the fraction of intermediate-band cells made available
for metric construction. Because the scored population changes with the
support fraction, comparisons across columns should be interpreted
descriptively. Methods within each column use the same scored cells.
}
\label{fig:pancreas_robustness}
\end{figure}

\begin{figure}[p]
\ContinuedFloat
\centering
\scriptsize
\renewcommand{\arraystretch}{0.92}

\textbf{(d) Conditional-path noise width and covariance}

\vspace{0.5mm}

\setlength{\tabcolsep}{3.1pt}
\resizebox{0.86\textwidth}{!}{%
\begin{tabular}{rrrrr}
\toprule
Width mult. &
$M_t=C_\rho(\mu_t)$ $W_2(1/2)$ &
$M_t=C_\rho(\mu_t)$ $W_2(1)$ &
$M_t=I$ $W_2(1/2)$ &
$M_t=I$ $W_2(1)$ \\
\midrule
0.003 & 10.4874 $\pm$ 0.0419 & 5.2457 $\pm$ 0.2301 & 10.4414 $\pm$ 0.1439 & 5.2980 $\pm$ 0.1595 \\
0.010 & 10.4940 $\pm$ 0.1123 & 5.5062 $\pm$ 0.4344 & 10.5187 $\pm$ 0.1591 & 5.2420 $\pm$ 0.2684 \\
0.030 & 10.4873 $\pm$ 0.0919 & 5.3305 $\pm$ 0.1619 & 10.4245 $\pm$ 0.1359 & 5.3180 $\pm$ 0.1351 \\
\textbf{0.100} & 10.3694 $\pm$ 0.0987 & 5.0662 $\pm$ 0.3449 & 10.3446 $\pm$ 0.1299 & 5.1521 $\pm$ 0.1702 \\
0.300 & 10.4630 $\pm$ 0.2142 & 5.3692 $\pm$ 0.1590 & 10.4149 $\pm$ 0.1922 & 5.2433 $\pm$ 0.2451 \\
1.000 & 14.0089 $\pm$ 0.0636 & 7.6047 $\pm$ 0.3015 & 13.8565 $\pm$ 0.1498 & 7.3363 $\pm$ 0.1651 \\
3.000 & 36.3654 $\pm$ 0.5652 & 24.9668 $\pm$ 0.7730 & 33.9906 $\pm$ 1.1181 & 26.6449 $\pm$ 1.1385 \\
10.000 & 118.0193 $\pm$ 4.7889 & 85.1409 $\pm$ 11.5354 & 143.4577 $\pm$ 2.9533 & 130.1162 $\pm$ 4.8789 \\
\bottomrule
\end{tabular}%
}

\vspace{3mm}

\begin{minipage}[t]{0.43\textwidth}
\centering
\textbf{(e) Path and inference covariance}

\vspace{0.5mm}

\setlength{\tabcolsep}{2.2pt}
\resizebox{\linewidth}{!}{%
\begin{tabular}{llrr}
\toprule
Path & Inference & $W_2(1/2)$ & $W_2(1)$ \\
\midrule
$C_\rho(\mu_t)$ & $C_\rho(X_t)$ & 10.3694 $\pm$ 0.0987 & 5.0662 $\pm$ 0.3449 \\
$C_\rho(\mu_t)$ & $I$ & 10.3617 $\pm$ 0.1003 & \textbf{4.9788 $\pm$ 0.2102} \\
$I$ & $C_\rho(X_t)$ & 10.3557 $\pm$ 0.1112 & 5.1712 $\pm$ 0.2807 \\
$I$ & $I$ & \textbf{10.3446 $\pm$ 0.1299} & 5.1521 $\pm$ 0.1702 \\
\bottomrule
\end{tabular}%
}
\end{minipage}
\hfill
\begin{minipage}[t]{0.55\textwidth}
\centering
\textbf{(f) Inference-noise schedule}

\vspace{0.5mm}

\setlength{\tabcolsep}{2.2pt}
\resizebox{\linewidth}{!}{%
\begin{tabular}{lrrrr}
\toprule
$g(t)$ &
$M(X_t)=C_\rho(X_t)$ $W_2(1/2)$ &
$M(X_t)=C_\rho(X_t)$ $W_2(1)$ &
$M(X_t)=I$ $W_2(1/2)$ &
$M(X_t)=I$ $W_2(1)$ \\
\midrule
$\sigma$ & 10.3694 $\pm$ 0.0987 & 5.0662 $\pm$ 0.3449 & 10.3617 $\pm$ 0.1003 & 4.9788 $\pm$ 0.2102 \\
$\propto (1-t)$ & 10.3580 & 5.0072 & 10.3496 & 5.0492 \\
$g\sim t$ & 10.3922 & 5.1977 & 10.3868 & 5.1169 \\
$g\sim \sqrt{t(1-t)}$ & 10.3629 & 5.0814 & 10.3515 & 5.0035 \\
$0$ after $t=0.8$ & 10.3626 & 5.0440 & 10.3543 & 4.9913 \\
$0$ & 10.4173 & 5.1955 & 10.4173 & 5.1955 \\
\bottomrule
\end{tabular}%
}
\end{minipage}

\vspace{1mm}

\caption{
\textbf{Pancreas robustness analyses (continued).}
(d) Sensitivity of stochastic FFM at $d=20$ to conditional-path noise
width and covariance.
(e) Sensitivity to conditional-path and inference covariance.
(f) Sensitivity to the scalar inference-noise schedule.
Values are mean $\pm$ population standard deviation over three seeds
unless stated otherwise. Lower is better.
}
\end{figure}

\begin{table}[t]
\centering
\label{tab:pancreas_runtimi}
\small
\caption{Wall-clock time in seconds per seed on a Tesla V100-SXM2-32GB on the Pancreas dataset.
The stochastic FFM arm reuses the deterministic interpolant and coupling.}
\begin{tabular}{lrrrr}
\toprule
Method & $d=2$ & $d=10$ & $d=20$ & $d=50$ \\
\midrule
OT-CFM & 3 & 3 & 3 & 3 \\
MFM / LAND & 22 & 61 & 109 & 254 \\
Curly-FM & 41 & 106 & 190 & 445 \\
FFM (ours), deterministic & 26 & 63 & 109 & 256 \\
FFM (ours), with noise & 15 & 29 & 48 & 108 \\
\bottomrule
\end{tabular}
\end{table}
\begin{table}[t]
\centering
\caption{One-at-a-time sensitivity of FFM on Pancreas. The reported value is the
validation selection objective; lower is better. The selected value is shown in bold.}
\label{tab:pancreas_sensitivity}
\small
\resizebox{\linewidth}{!}{
\begin{tabular}{cllll}
\toprule
$d$ & Varied parameter & Fixed parameter & Values & Validation objective \\
\midrule
2
& $\rho$
& $\lambda=0.1$
& $0.01,\ 0.03,\ \mathbf{0.1},\ 0.3,\ 1.0$
& $0.828,\ 0.749,\ \mathbf{0.535},\ 0.870,\ 1.131$ \\

2
& $\lambda$
& $\rho=0.1$
& $0.01,\ 0.03,\ \mathbf{0.1},\ 0.3,\ 1.0$
& $1.226,\ 0.951,\ \mathbf{0.535},\ 0.826,\ 0.873$ \\
\midrule

20
& $\rho$
& $\lambda=1.0$
& $0.003,\ 0.01,\ \mathbf{0.03},\ 0.1,\ 0.3$
& $8.180,\ 8.210,\ \mathbf{8.135},\ 8.369,\ 8.288$ \\

20
& $\lambda$
& $\rho=0.03$
& $0.03,\ 0.1,\ 0.3,\ \mathbf{1.0},\ 3.0$
& $8.990,\ 8.535,\ 8.370,\ \mathbf{8.135},\ 8.667$ \\
\midrule

50
& $\rho$
& $\lambda=0.03$
& $3{\times}10^{-4},\ 10^{-3},\ \mathbf{3{\times}10^{-3}},\ 10^{-2},\ 3{\times}10^{-2}$
& $10.455,\ 10.421,\ \mathbf{10.314},\ 10.390,\ 10.589$ \\

50
& $\lambda$
& $\rho=0.003$
& $0.003,\ 0.01,\ \mathbf{0.03},\ 0.1,\ 0.3$
& $10.379,\ 10.314,\ \mathbf{10.314},\ 10.554,\ 10.537$ \\
\bottomrule
\end{tabular}}
\end{table}
\paragraph{Erythroid:}\label{app: Erythroid}
We further evaluate FFM on the mouse erythroid differentiation dataset
used by \citet{petrovicCurlyFlowMatching2025a}. The dataset consists of
9,815 cells along the erythroid developmental trajectory, partitioned into
three temporal snapshots with the intermediate population withheld during
training. We evaluate in the two-dimensional embedding and in
$d\in\{20,50\}$ dimensional representations. In addition to our methods,
we report both the values originally published by
\citet{petrovicCurlyFlowMatching2025a} and our reproduction of the
baselines using a common training harness at their published training
budget.

\begin{figure*}[t]
\centering

\begin{minipage}[c]{0.38\textwidth}
    \centering
    \includegraphics[width=\linewidth]{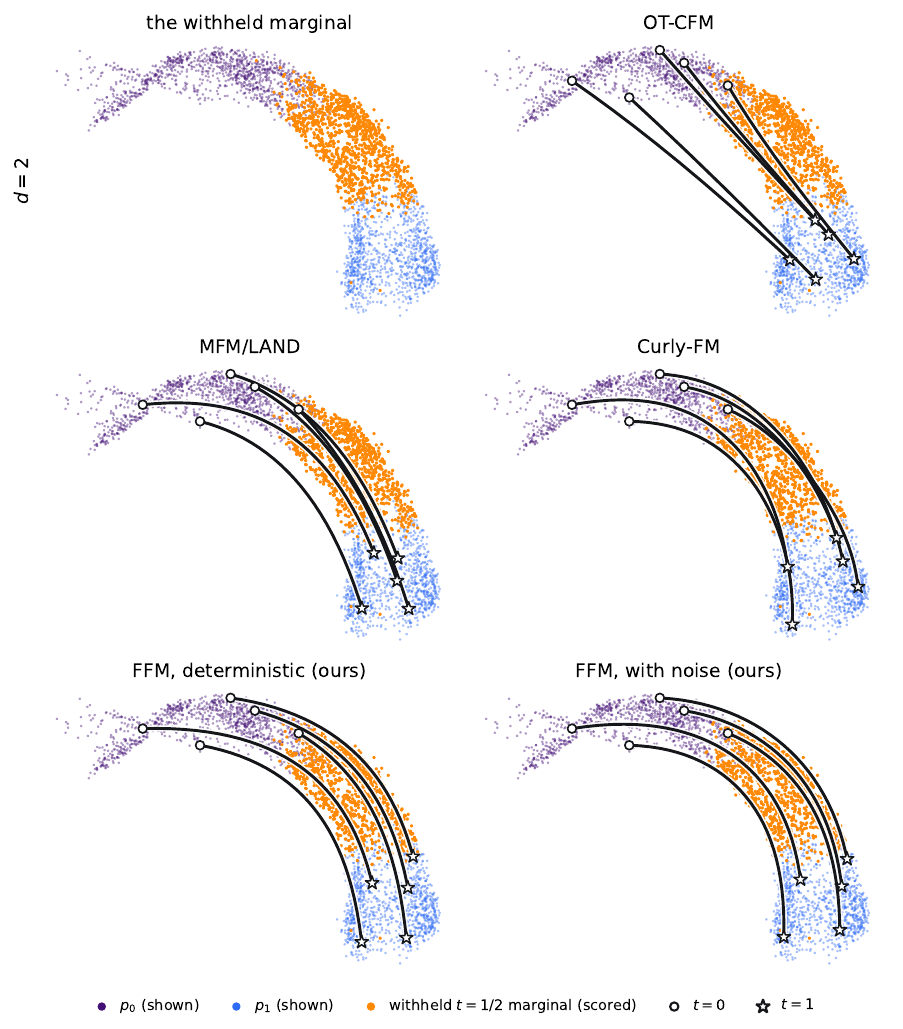}
\end{minipage}
\hfill
\begin{minipage}[c]{0.60\textwidth}
    \centering
    \small
    \setlength{\tabcolsep}{3.5pt}

    \resizebox{\linewidth}{!}{%
    \begin{tabular}{llrrrr}
\toprule
Method & Source & Cos. Dist. $\downarrow$ & $L_2\downarrow$ & $W_1\downarrow$ & $W_2\downarrow$ \\
\midrule
\multicolumn{6}{l}{\textit{$d=2$}} \\
OT-CFM & Published & 0.146 $\pm$ 0.001 & 2.704 $\pm$ 0.019 & 0.646 $\pm$ 0.006 & -- \\
OT-CFM & Reproduced & 0.141 $\pm$ 0.002 & 2.447 $\pm$ 0.031 & 0.645 $\pm$ 0.009 & 0.679 $\pm$ 0.009 \\
MFM (LAND) & Published & 0.014 $\pm$ 0.001 & 1.999 $\pm$ 0.014 & 0.269 $\pm$ 0.004 & -- \\
MFM (LAND) & Reproduced & \textbf{0.010 $\pm$ 0.001} & 1.858 $\pm$ 0.025 & 0.271 $\pm$ 0.007 & 0.286 $\pm$ 0.007 \\
Curly-FM & Published & 0.009 $\pm$ 0.000 & 1.663 $\pm$ 0.293 & 0.369 $\pm$ 0.090 & -- \\
Curly-FM & Reproduced & 0.011 $\pm$ 0.001 & 2.151 $\pm$ 0.078 & \textbf{0.122 $\pm$ 0.002} & \textbf{0.158 $\pm$ 0.003} \\
FFM (ours) & Ours & 0.021 $\pm$ 0.004 & \textbf{1.642 $\pm$ 0.071} & 0.248 $\pm$ 0.034 & 0.291 $\pm$ 0.023 \\
FFM (ours), with noise & Ours & 0.019 $\pm$ 0.003 & 1.710 $\pm$ 0.081 & 0.226 $\pm$ 0.027 & 0.269 $\pm$ 0.019 \\
\midrule
\multicolumn{6}{l}{\textit{$d=20$}} \\
OT-CFM & Published & 0.489 $\pm$ 0.001 & 1.885 $\pm$ 0.020 & 6.103 $\pm$ 0.074 & -- \\
OT-CFM & Reproduced & \textbf{0.490 $\pm$ 0.000} & 1.731 $\pm$ 0.014 & 6.106 $\pm$ 0.063 & 6.260 $\pm$ 0.066 \\
MFM (LAND) & Published & 0.495 $\pm$ 0.001 & 1.627 $\pm$ 0.040 & 4.855 $\pm$ 0.052 & -- \\
MFM (LAND) & Reproduced & 0.494 $\pm$ 0.000 & 1.700 $\pm$ 0.021 & 4.759 $\pm$ 0.056 & 4.855 $\pm$ 0.054 \\
Curly-FM & Published & 0.488 $\pm$ 0.001 & 1.721 $\pm$ 0.035 & 6.124 $\pm$ 0.027 & -- \\
Curly-FM & Reproduced & \textbf{0.490 $\pm$ 0.000} & 1.691 $\pm$ 0.011 & 6.112 $\pm$ 0.053 & 6.271 $\pm$ 0.058 \\
FFM (ours) & Ours & 0.491 $\pm$ 0.000 & 1.514 $\pm$ 0.016 & \textbf{4.685 $\pm$ 0.012} & \textbf{4.782 $\pm$ 0.014} \\
FFM (ours), with noise & Ours & 0.496 $\pm$ 0.000 & \textbf{0.757 $\pm$ 0.006} & 4.733 $\pm$ 0.035 & 4.823 $\pm$ 0.038 \\
\midrule
\multicolumn{6}{l}{\textit{$d=50$}} \\
OT-CFM & Published & 0.490 $\pm$ 0.000 & 2.215 $\pm$ 0.022 & 7.969 $\pm$ 0.029 & -- \\
OT-CFM & Reproduced & \textbf{0.491 $\pm$ 0.000} & 2.053 $\pm$ 0.017 & 7.839 $\pm$ 0.016 & 7.974 $\pm$ 0.016 \\
MFM (LAND) & Published & 0.494 $\pm$ 0.000 & 1.971 $\pm$ 0.023 & 6.727 $\pm$ 0.022 & -- \\
MFM (LAND) & Reproduced & 0.495 $\pm$ 0.000 & 2.004 $\pm$ 0.013 & 6.697 $\pm$ 0.052 & 6.819 $\pm$ 0.058 \\
Curly-FM & Published & 0.489 $\pm$ 0.000 & 2.045 $\pm$ 0.073 & 7.729 $\pm$ 0.046 & -- \\
Curly-FM & Reproduced & \textbf{0.491 $\pm$ 0.000} & 1.990 $\pm$ 0.021 & 7.736 $\pm$ 0.025 & 7.864 $\pm$ 0.025 \\
FFM (ours) & Ours & \textbf{0.491 $\pm$ 0.000} & 1.772 $\pm$ 0.029 & 6.810 $\pm$ 0.022 & 6.907 $\pm$ 0.025 \\
FFM (ours), with noise & Ours & 0.499 $\pm$ 0.000 & \textbf{0.915 $\pm$ 0.006} & \textbf{6.595 $\pm$ 0.016} & \textbf{6.663 $\pm$ 0.015} \\
\bottomrule
\end{tabular}%
    }
\end{minipage}

\caption{
Visualization of the learned erythroid trajectories in 2D and quantitative
comparison across dimensions $d=2,20,50$.
``Published'' denotes the values reported in Table~4 of
\citet{petrovicCurlyFlowMatching2025a}, while ``Reproduced'' denotes our
reproduction of the corresponding baselines using a common training harness
at the published training budget.
Due to a bug in the original Curly-FM evaluation, the Wasserstein quantity
reported as $W_2$ in Table~4 of CurlyFM
corresponds in fact to $W_1$. We therefore report these published values
under $W_1$ and additionally compute the true $W_2$ distance for all runs
evaluated with our harness.
For $d=20$ and $d=50$, $L_2$ is reported in units of $10^3$.
Bold values indicate the best performance among methods evaluated with our
common harness; published values are shown for reference and are not included
when determining the best result.
}
\label{fig:erythroid_paths_and_quantitative_results}
\end{figure*}
\paragraph{iTracer:} Let $\mathcal N_{15}(i)$ denote the 15 nearest neighbours of
cell $i$, excluding itself, and let $j_7(i)$ denote its seventh
nearest neighbour. All four kernel variants share the same
directed, unsymmetrized support and self-tuning bandwidths.
The expression factor is
\begin{equation}
    K^{\mathrm{expr}}_{ij}
    =
    \mathbf{1}_{\{j\in\mathcal N_{15}(i)\}}
    \exp\!\left(
        -\frac{\|x_i-x_j\|^2}{\sigma_i\sigma_j}
    \right),
    \qquad
    \sigma_i=\|x_i-x_{j_7(i)}\|.
    \label{eq:itracer-expression}
\end{equation}

For barcode-family labels $\mathrm{bf}_i$ and scar-family
labels $\mathrm{sf}_i$, define
\begin{equation}
    R_{ij}
    =
    1-\frac{1}{2}\left(
        \mathbf{1}_{\{\mathrm{bf}_i\ne\mathrm{bf}_j\}}
        +
        \mathbf{1}_{\{\mathrm{sf}_i\ne\mathrm{sf}_j\}}
    \right),
    \qquad
    K^{\mathrm{lin}}_{ij}
    =
    \exp(\beta_{\mathrm{lin}}R_{ij}).
    \label{eq:itracer-lineage}
\end{equation}
Pairs involving a cell without a barcode call receive
$R_{ij}=0$. Barcode and scar families containing fewer than
100 and 50 cells, respectively, are excluded. For the RNA-velocity vector $v_i$, the velocity factor is
\begin{equation}
    a_{ij}
    =
    \frac{\langle v_i,x_j-x_i\rangle}
         {\|v_i\|\,\|x_j-x_i\|},
    \qquad
    K^{\mathrm{vel}}_{ij}
    =
    \exp(\nu a_{ij}).
    \label{eq:itracer-velocity}
\end{equation} 
The final transition matrix is
\begin{equation}
    W_{ij}
    =
    K^{\mathrm{expr}}_{ij}
    K^{\mathrm{lin}}_{ij}
    K^{\mathrm{vel}}_{ij},
    \qquad
    P_{ij}
    =
    \frac{W_{ij}}{\sum_\ell W_{i\ell}},
    \qquad
    P_{ii}=0.
    \label{eq:itracer-transition}
\end{equation}

The different choices
\[
    (\beta_{\mathrm{lin}},\nu)
    \in\{(0,0),(0,3),(3,0),(3,3)\}
\]
correspond to expression only, expression plus velocity,
expression plus lineage, and all three factors, respectively.

\begin{figure}[t]
    \centering
    \includegraphics[width=0.85\linewidth]
        {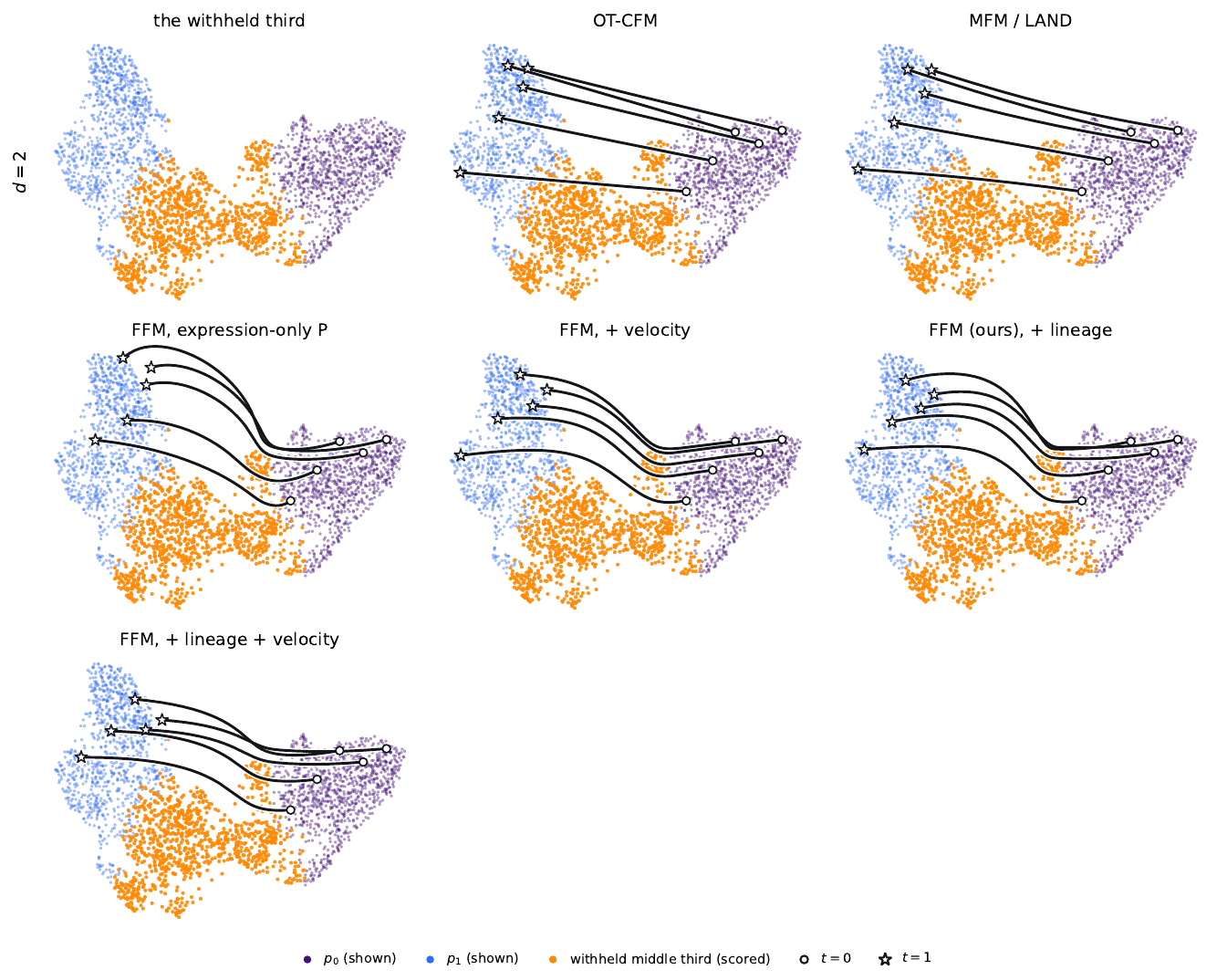}
    \caption{Learned trajectories on iTracer in the
    two-dimensional representation.}
    \label{fig:itracer-paths-appendix}
\end{figure}
\paragraph{Bifurcation:}
We evaluate a synthetic bifurcation comprising $3000$ points in $\mathbb R^3$, where one source population splits into two target arms and the middle split is withheld \cref{fig:forksheet-appendix}.
Hyperparameters are transferred from Sheet without retuning.
To evaluate cell-level fate preservation, we integrate the
324 of 969 source cells with resolved initial arm labels and
classify their endpoints as lower arm, upper arm, or outside.
The target regions are constructed from balls of radius $1/2$ of the distance between arm centres
Table~\cref{tab:forksheet_fate}. 
FFM maintains high recall on both arms, whereas OT-CFM and
Curly-FM misassign lower-arm cells, and MFM loses upper-arm
recall primarily through outside outcomes.
A separate allocation sweep
\cref{tab:forksheet_allocation_recall}
shows that performance under imbalance remains
configuration-dependent.
\begin{figure}[t]
    \centering
    \begin{subfigure}[c]{0.20\linewidth}
        \centering
        \includegraphics[width=\linewidth]
            {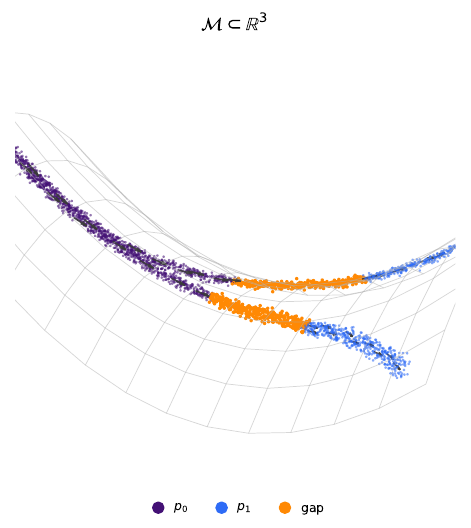}
        \caption{Dataset.}
    \end{subfigure}\hfill
    \begin{subfigure}[c]{0.78\linewidth}
        \centering
        \includegraphics[width=\linewidth]
            {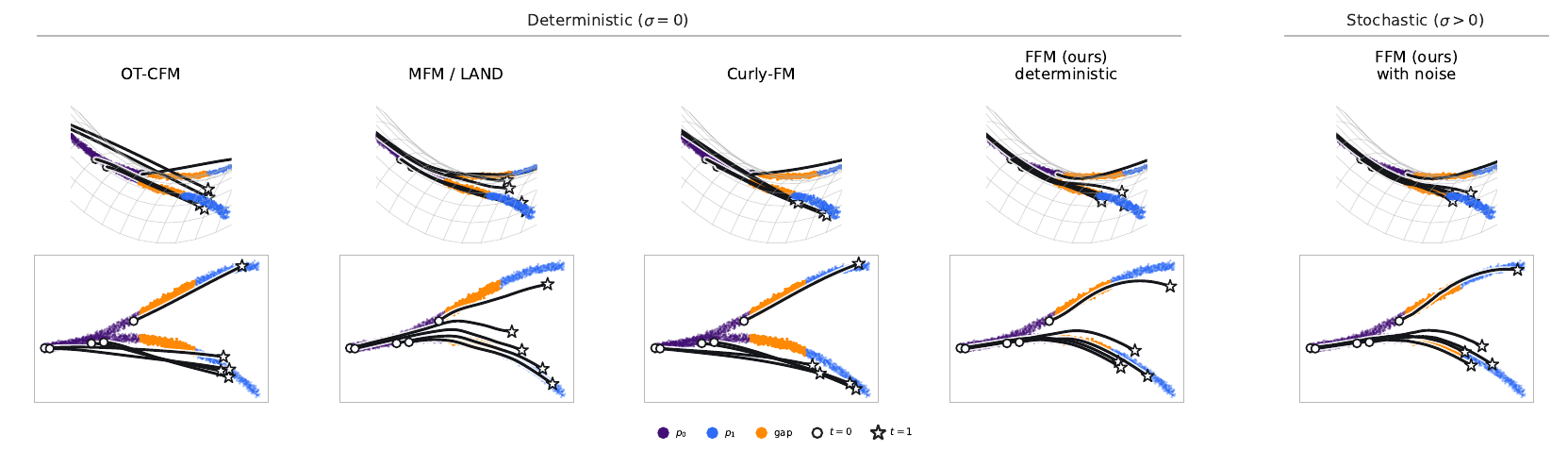}
        \caption{Learned trajectories.}
    \end{subfigure}
    \caption{ForkSheet benchmark.}
    \label{fig:forksheet-appendix}
\end{figure}
\begin{table}[t]
\centering
\caption{ForkSheet fate preservation. Values are percentages, reported as
mean $\pm$ population standard deviation over three seeds.}
\label{tab:forksheet_fate}
\vspace{4pt}
\small
\setlength{\tabcolsep}{4pt}
\renewcommand{\arraystretch}{1.1}
\begin{adjustbox}{max width=\linewidth}
\begin{tabular}{llrrr}
\toprule
Method & True arm & Recall & Wrong arm & Outside \\
\midrule
OT-CFM & Lower & 69.4 $\pm$ 10.7 & 30.6 $\pm$ 10.7 & 0.0 $\pm$ 0.0 \\
& Upper & 100.0 $\pm$ 0.0 & 0.0 $\pm$ 0.0 & 0.0 $\pm$ 0.0 \\
MFM / LAND & Lower & 91.4 $\pm$ 5.6 & 4.4 $\pm$ 3.9 & 4.2 $\pm$ 3.4 \\
& Upper & 68.1 $\pm$ 37.3 & 12.3 $\pm$ 17.4 & 19.6 $\pm$ 20.1 \\
Curly-FM & Lower & 48.8 $\pm$ 16.7 & 47.4 $\pm$ 15.2 & 3.8 $\pm$ 1.5 \\
& Upper & 100.0 $\pm$ 0.0 & 0.0 $\pm$ 0.0 & 0.0 $\pm$ 0.0 \\
FFM (ours), deterministic & Lower & 75.5 $\pm$ 19.7 & 21.2 $\pm$ 19.4 & 3.4 $\pm$ 0.3 \\
& Upper & 100.0 $\pm$ 0.0 & 0.0 $\pm$ 0.0 & 0.0 $\pm$ 0.0 \\
FFM (ours), with noise & Lower & 68.3 $\pm$ 19.8 & 20.8 $\pm$ 22.7 & 10.9 $\pm$ 12.7 \\
& Upper & 98.8 $\pm$ 0.9 & 0.0 $\pm$ 0.0 & 1.2 $\pm$ 0.9 \\
\bottomrule
\end{tabular}

\end{adjustbox}
\end{table}
\begin{table}[t]
\centering
\caption{ForkSheet recall under different nominal lower/upper fate
allocations. Values are percentages, reported as mean $\pm$ population
standard deviation over three seeds. Within each cell, the first line
reports lower-arm recall and the second line reports upper-arm recall.
These runs are separate from the evaluation in Table~\cref{tab:forksheet_fate}.}
\label{tab:forksheet_allocation_recall}
\vspace{4pt}
\small
\setlength{\tabcolsep}{3pt}
\renewcommand{\arraystretch}{1.1}
\begin{adjustbox}{max width=\linewidth}
\begin{tabular}{lccccc}
\toprule
Method & 50/50 & 70/30 & 40/60 & 20/80 & 10/90 \\
\midrule
OT-CFM
& \begin{tabular}[c]{@{}c@{}}69.4 $\pm$ 10.7\\100.0 $\pm$ 0.0\end{tabular}
& \begin{tabular}[c]{@{}c@{}}71.1 $\pm$ 2.5\\100.0 $\pm$ 0.0\end{tabular}
& \begin{tabular}[c]{@{}c@{}}57.6 $\pm$ 11.7\\99.5 $\pm$ 0.7\end{tabular}
& \begin{tabular}[c]{@{}c@{}}59.0 $\pm$ 17.4\\99.1 $\pm$ 0.6\end{tabular}
& \begin{tabular}[c]{@{}c@{}}36.1 $\pm$ 16.4\\98.4 $\pm$ 1.3\end{tabular} \\
MFM / LAND
& \begin{tabular}[c]{@{}c@{}}91.2 $\pm$ 9.5\\76.0 $\pm$ 16.6\end{tabular}
& \begin{tabular}[c]{@{}c@{}}99.9 $\pm$ 0.2\\84.8 $\pm$ 1.8\end{tabular}
& \begin{tabular}[c]{@{}c@{}}92.0 $\pm$ 8.3\\98.5 $\pm$ 2.2\end{tabular}
& \begin{tabular}[c]{@{}c@{}}73.2 $\pm$ 23.6\\74.0 $\pm$ 34.9\end{tabular}
& \begin{tabular}[c]{@{}c@{}}50.9 $\pm$ 15.1\\96.9 $\pm$ 4.4\end{tabular} \\
Curly-FM
& \begin{tabular}[c]{@{}c@{}}48.8 $\pm$ 16.7\\100.0 $\pm$ 0.0\end{tabular}
& \begin{tabular}[c]{@{}c@{}}68.0 $\pm$ 7.1\\100.0 $\pm$ 0.0\end{tabular}
& \begin{tabular}[c]{@{}c@{}}37.2 $\pm$ 3.9\\100.0 $\pm$ 0.0\end{tabular}
& \begin{tabular}[c]{@{}c@{}}33.9 $\pm$ 18.8\\100.0 $\pm$ 0.0\end{tabular}
& \begin{tabular}[c]{@{}c@{}}0.0 $\pm$ 0.0\\100.0 $\pm$ 0.0\end{tabular} \\
FFM (ours), deterministic
& \begin{tabular}[c]{@{}c@{}}95.6 $\pm$ 3.4\\94.1 $\pm$ 5.2\end{tabular}
& \begin{tabular}[c]{@{}c@{}}92.2 $\pm$ 10.4\\89.1 $\pm$ 13.9\end{tabular}
& \begin{tabular}[c]{@{}c@{}}89.4 $\pm$ 5.7\\94.2 $\pm$ 8.2\end{tabular}
& \begin{tabular}[c]{@{}c@{}}98.9 $\pm$ 0.8\\97.7 $\pm$ 0.8\end{tabular}
& \begin{tabular}[c]{@{}c@{}}77.8 $\pm$ 10.4\\98.3 $\pm$ 0.3\end{tabular} \\
FFM (ours), with noise
& \begin{tabular}[c]{@{}c@{}}97.3 $\pm$ 3.4\\94.5 $\pm$ 7.7\end{tabular}
& \begin{tabular}[c]{@{}c@{}}94.0 $\pm$ 8.5\\81.2 $\pm$ 19.5\end{tabular}
& \begin{tabular}[c]{@{}c@{}}93.8 $\pm$ 4.2\\98.5 $\pm$ 1.8\end{tabular}
& \begin{tabular}[c]{@{}c@{}}99.5 $\pm$ 0.8\\98.7 $\pm$ 0.4\end{tabular}
& \begin{tabular}[c]{@{}c@{}}79.6 $\pm$ 6.5\\93.1 $\pm$ 4.8\end{tabular} \\
\bottomrule
\end{tabular}
\end{adjustbox}
\end{table}
\section{Related Work}
\paragraph{Relation to Metric Flow Matching.}
Metric Flow Matching (MFM) \citep{kapusniakMetricFlowMatchinga} learns
endpoint-conditioned approximate geodesics of a data-dependent Riemannian metric $g(x)=G(x;\mathcal D)$ (Definition 1) as interpolants. This construction is contained in FFM as the symmetric special case in which no directional information is used. More precisely $G(x;\mathcal D)$ is a special case of a $P-$derived Randers with $F_P(x,v) = \sqrt{v^T G(x;\mathcal D) v}$. 
\\
\\
There is also an explicit connection between the LAND metric \citep{arvanitidisLocallyAdaptiveNormal2016} considered in
MFM and our $P$-derived metric~\eqref{eq:empirical-fw-metric}. Consider a
soft $k$-NN kernel
\begin{align*}
    w_{ij}
    =
    \mathbf{1}_{\{j\in\mathcal{N}_k(i)\}}
    \exp\!\left(
        -\frac{\|x_i-x_j\|^2}{2\sigma^2}
    \right),
    \qquad
    P_{ij}
    =
    \frac{w_{ij}}{Z_i},
    \qquad
    Z_i=\sum_j w_{ij}.
\end{align*}
The corresponding second transition moment is
\begin{align*}
    C_N(x_i)
    =
    \frac{1}{Z_i}
    \sum_j
    w_{ij}
    (x_j-x_i)(x_j-x_i)^\top .
\end{align*}
On the other hand, defining
$h_\alpha(x_i)=\sum_jw_{ij}(x_j^\alpha-x_i^\alpha)^2,$
we obtain
\begin{equation}
    \operatorname{diag} C_N(x_i)
    =
    \frac{1}{Z_i}\operatorname{diag}(h(x_i)).
\end{equation}
Consequently,
\begin{equation}
    \left(
        \operatorname{diag}C_N(x_i)+\rho I
    \right)^{-1}
    =
    Z_i
    \left(
        \operatorname{diag}(h(x_i))
        +\rho Z_i I
    \right)^{-1}.
\end{equation}
Thus, apart from the normalization $Z_i$ induced by requiring $P$ to be
row-stochastic, the diagonal of our second-moment geometry has exactly the
form of the LAND metric on datapoints. In particular, when $Z_i$ is approximately constant
over the region of interest, choosing $\varepsilon=\rho Z$ recovers the LAND
metric up to the global multiplicative factor $Z$, which does not change its
geodesic curves.
\\
\\
FFM can therefore be viewed as a generalization of MFM incoporating a sense of directionality when given an transition probability matrix $P$.
\paragraph{Relation to Curly FM:}
Curly Flow Matching (Curly-FM) \citep{petrovicCurlyFlowMatching2025a} assumes access to a reference vector field $f_t(x)$ and learns endpoint conditioned paths whose velocities remain close to this field, using an objective of the form
\begin{align*}
    \mathcal{L}_{\mathrm{Curly}}(\eta)=\mathbb{E}_{x_0,x_1,t}\left[\left\|\dot{x}_{t,\eta}-f_t(x_{t,\eta})\right\|_2^2 \right].
\end{align*}
When the dynamics are instead specified by a transition matrix $P$, the
natural first-order vector-field summary is its first spatial moment $m(x_i)=\sum_{j=1}^N P_{ij}(x_j-x_i),$. If we replace the
Euclidean norm in the Curly objective by the second moment $C_{\rho}^{-1}$
\begin{align}\label{fixed-time FW}
    \left\|\dot{x}_{t,\eta}-m(x)\right\|_{C(x)^{-1}}^2 = (\dot{x}_{t,\eta}-m(x))^T C_\rho(x)^{-1}(\dot{x}_{t,\eta}-m(x))
\end{align}
which gives us the integrant of the fixed time FW functional \cref{eq: Fixed-Time FW}. For data that follows the manifold hypothesis we have in particular that $C\neq I$. This observation also suggests a direct second-moment generalization of the Curly-FM path objective. Specifically, we replace its Euclidean
velocity residual by the graph-derived anisotropic residual
\[
\frac{1}{a}
\bigl(\dot x-a\,m(x)\bigr)^\top
C_\rho(x)^{-1}
\bigl(\dot x-a\,m(x)\bigr),
\]
so that the control has access to the same first and second transition
moments as FFM.
We evaluate this matched control on Pancreas in
Fig.~\cref{fig:paths_and_quantitative_results}. Full FFM obtains lower
$W_2(1/2)$ under both a shared endpoint coupling and method-specific
couplings, indicating that the observed improvement is not explained
solely by supplying the path objective with second-moment information.
\\
\\
Furthermore, FFM can also be used in settings with access to a reference vector field. Following velocity-kernel constructions such as those used in CellRank \citep{langeCellRankDirectedSinglecell2022a}, a vector field $f$ can be converted into a local
transition matrix $P^f$. Let
$W^{(h)}_{ij}=\mathbf{1}\!\left\{\|x_j-x_i\|\le h_i\right\},$
where $h_i$ denotes the distance to the $k$-th nearest neighbor, and define
\begin{align*}
P^f_{ij}
=
\frac{
W^{(h)}_{ij}
\exp\!\left(
\kappa
\frac{
f(x_i)^\top(x_j-x_i)
}{
\|f(x_i)\|\,\|x_j-x_i\|
}
\right)
}{
\sum_\ell
W^{(h)}_{i\ell}
\exp\!\left(
\kappa
\frac{
f(x_i)^\top(x_\ell-x_i)
}{
\|f(x_i)\|\,\|x_\ell-x_i\|
}
\right)
}.
\end{align*}
For dense and locally isotropic sampling, the $k$-NN neighborhood of $x_i$
is approximately rotationally symmetric. The cosine kernel assigns a
displacement direction $u$ a weight proportional to
$\exp(\kappa\,\hat f(x_i)^\top u)$, where
$\hat f=f/\|f\|$. By local isotropy, contributions orthogonal to $\hat f$ appear in pairs in opposite directions with and hence cancel eachother. Consequently,
\begin{align*}
    m_N^f(x_i)
    =
    \sum_j P^f_{ij}(x_j-x_i)
    =
    c_i\,\hat f(x_i) + o(1),
    \qquad c_i>0,
\end{align*}
as the neighborhood becomes dense. Thus, the first moment of the cosine
velocity kernel recovers the orientation of the vector field. FFM can therefore be applied also when the dynamical information is given as a vector field: its first moment retains the directional information used by Curly-FM, while its second moment additionally describes the local distribution of accessible transitions. This is particularly helpful when the vector field measurement is noisy, as is often the case for RNA-velocity estimates. Then the construction of the transition kernel can be viewed as a denoising of the vector field.

\paragraph{Multimarginal Flow Matching}
\end{appendix}